\documentclass[12pt,twoside]{article}
\usepackage{geometry}
\usepackage[labelfont=bf]{caption}
\usepackage[numbers]{natbib}
\usepackage{amssymb}
\usepackage{graphicx}
\usepackage{float}
\usepackage{booktabs}
\usepackage{tabularx}
\usepackage{titling}
\usepackage{amsmath}
\usepackage{amsthm}
\usepackage{dsfont}
\usepackage{titlesec}
\usepackage{subcaption}
\usepackage{breqn}
\usepackage{thmtools}

\usepackage{algorithm}
\usepackage{algorithmic}

\titlespacing*{\section}
  {0pt}{-.1\baselineskip}{-.1\baselineskip}  
\titlespacing*{\subsection}
   {0pt}{-.1\baselineskip}{-.1\baselineskip}  

\usepackage[T1]{fontenc}

\date{}
\providecommand{\keywords}[1]
{
   \small	
  \textit{\hspace{-1em} Keywords: } #1
}

\title{ \textbf{Online Learning Approach for Survival Analysis}} 
\author{\normalsize Camila Fernandez \and \normalsize Pierre Gaillard \and \normalsize Joseph de Vilmarest \and \normalsize Olivier Wintenberger}

\usepackage{mathtools}

\newtheorem{Theorem}{Theorem}
\newtheorem{assumption}{Assumption}
\newtheorem{Lemma}{Lemma}
\newtheorem{Definition}{Definition}
\newtheorem{proposition}{Proposition}

\newtheorem{Corollary}{Corollary}

\usepackage{titlesec}
\titleformat{\section}[block]
  {\fontsize{12}{15}\bfseries\sffamily\filcenter}
  {\thesection}
  {1em}
  {\MakeUppercase}
\titleformat{\subsection}[hang]
  {\fontsize{12}{15}\bfseries\sffamily}
  {\thesubsection}
  {1em}
  {}

\begin{document}

\maketitle

\vspace{-6em}

\begin{abstract}
\noindent  We introduce an online mathematical framework for survival analysis, allowing real time adaptation to dynamic environments and censored data. This framework enables the estimation of event time distributions through an optimal second order online convex optimization algorithm—Online Newton Step (ONS). This approach, previously unexplored, presents substantial advantages, including explicit algorithms with non-asymptotic convergence guarantees. Moreover, we analyze the selection of ONS hyperparameters, which depends on the exp-concavity property and has a significant influence on the regret bound. We propose a stochastic approach that guarantees logarithmic stochastic regret for ONS. Additionally, we introduce an adaptive aggregation method that ensures robustness in hyperparameter selection while maintaining fast regret bounds. The findings of this paper can extend beyond the survival analysis field, and are relevant for any case characterized by poor exp-concavity and unstable ONS. Finally, these assertions are illustrated by simulation experiments.\\
\keywords{\textit{Online learning, survival analysis, regret bounds, convex optimization, stochastic risk.}}  
\end{abstract}

\section{Introduction}
On the one hand the primary objective of survival analysis is to estimate the time until a critical event occurs, often referred to as survival time or failure time. Examples of such events include customer churn \cite{lu2002predicting}, machine failures \cite{chen2020predictive}, and employees' attrition \cite{mohbey2022employee}. Survival analysis is particularly suited for scenarios where the occurrence of the event may not be observed for all individuals in the dataset. This phenomenon arises when data collection happened before the event occurred, or individuals left the study before experiencing the event, and is called right censoring. As neglecting the censored data is restrictive, it is essential to consider censorship in estimating event time distributions to avoid bias and underestimation. For each individual $i$ with event time $t_i$, we define the survival probability function as
\[
S_i(t) = \mathbb P (t_i \ge t), \qquad t\ge 0.
\]

On the other hand convex optimization aims to find the minimum of a convex function over a convex set. It can be extended to an online approach in which the dataset becomes available in sequential order and is used to update the estimations of the algorithms at each step. This setting is suitable when the dataset is rapidly evolving over time, allowing for efficient processing of large volumes of data. Online convex optimization is a broad field with diverse applications such as online portfolio selection in finance, signal processing, communication, and machine learning algorithms; see Hazan \cite{hazan2022introduction} and references therein.

In this paper, we propose the application of online convex optimization algorithms to survival analysis. The combination of these two approaches has not been explored before. Our method offers significant advantages, including explicit algorithms with non-asymptotic convergence guarantees, making it a promising tool for the survival analysis field. 

Specifically, we estimate a parametric survival probability function $S_i$ using online convex optimization algorithms: let $\Theta$ be a non-empty, convex, compact set in $\mathbb R^d$, and $\ell_t$ the negative log-likelihood of the individuals at risk during the interval $(t-1,t]$, $t\ge 1$. The performance of online convex optimization algorithms is measured with the regret
\[
Regret_n := \displaystyle \sum_{t=1}^n \ell_{t}(\theta_t) - \min_{\theta \in \Theta} \sum_{t=1}^n \ell_{t}(\theta), \qquad n \ge 1,
\]
which indicates how close the cumulative loss is to the optimal solution. A smaller regret implies better performance, and our objective is to bound its growth with respect to $n$ as slowly as possible.

One of the most widely used algorithms in online convex optimization is the Online Newton Step (ONS) of Hazan et al. \cite{hazan2007logarithmic}, renowned for its fast regret convergence rate for exp-concave loss functions. This second-order algorithm relies on a hyperparameter known as the learning rate, whose optimal selection is directly dependent on the exp-concavity properties of the loss functions. The exp-concavity constant plays a fundamental role in the theoretical regret analysis of ONS.

We give a detailed mathematical framework for online survival analysis data and we implement the ONS method to optimize the negative log-likelihood of the exponential model. We note that the ONS algorithm requires a careful selection of the learning rate to ensure robust performance. However, certain choices, such as the learning rate proposed by Hazan et al. \cite{hazan2007logarithmic}, might lead to an explosive increase in regret, particularly when applied to the survival losses $\ell_t$. Therefore, proper selection of the learning rate is essential in our application.

We discuss various strategies for selecting the learning rate hyperparameter. The first contribution involves applying the stochastic setting from Wintenberger~\cite{wintenberger2021stochastic} to the survival case. This setting enhances convex properties by assessing stochastic risks rather than cumulative losses, allowing us to attain theoretical guarantees for the stochastic regret that is strongly related to the exp-concavity properties on average. Consequently, this provides the convergence of the algorithm estimations to the real parameter under well-specification. Secondly, in the deterministic setting, we propose to apply ONS to an auxiliary function that recursively adapts the learning rate in response to updates in the exp-concavity constant. We introduce the algorithm SurvONS, an aggregation procedure which ensures a logarithmic regret bound and robustness in hyperparameter selection over a fixed grid. This provides a new compromise in the context of second-order algorithms: the algorithm either performs well on average (as in the case of BOA \cite{wintenberger2017optimal}) or performs well for certain iterations (as in our case with SurvONS). It is important to emphasize that this algorithm is applicable not only to the survival case but also to any case where the exp-concavity properties are poor and the original versio of ONS is unstable. Finally, we conduct experiments using simulated data to examine our algorithm's behavior under different constraints. We discuss the choice of the grid, and we observe that the combination of multiple ONS allows us to use larger grids in SurvONS than in BOA-ONS~\cite{wintenberger2017optimal}.

The literature in survival analysis is considerable. The approaches range from non-parametric methods, such as the one proposed by Kaplan and Meier in 1958 \cite{kaplan1958nonparametric}, to semi-parametric methods like Cox proportional hazards \cite{cox1972regression}, and more recent machine learning applications. For instance, Ishwaran proposed an adapted random forest for censored data inx \cite{ishwaran2008random}. Another example is DeepSurv, which was introduced by Katzman in \cite{katzman2018deepsurv}. DeepSurv utilizes deep learning techniques to estimate the log-risk function in the Cox model. From a theoretical perspective, Arjas and Haara \cite{arjas1987logistic} proposed a dynamic setting called discrete-time logistic regression. In this model, events are always treated in the order in which they occurred in real time. The authors provided an asymptotic normality result for the maximum likelihood estimator of the regression coefficients. The discrete model is a suitable choice when events are observed at discrete time points; see Tutz \cite{tutz2016modeling}. Building upon Arjas and Haara's framework, Fahrmeir \cite{fahrmeir1994dynamic} introduced a state-space approach for analyzing discrete-time survival data. This approach includes the estimation of time-varying covariate effects achieved by maximizing posterior densities through the use of Kalman Filter algorithms. Christoffersen \cite{christoffersen2021dynamichazard} provided a method for discretising continuous event times when the instantaneous hazard follows an exponential shape. In a similar setting we provide adaptive estimators with non-asymptotic guarantees for the first time.

\section{Background on parametric inference} \label{paramsection}
\subsection{Notation}
We consider a set of $N$ individuals denoted by $i \in \{1, \ldots, N\}$, each associated with an arrival time $\tau_i \geq 0$. Such time could represent when a patient enters the hospital, a client joins the company, or simply when an individual enrolls in the study. Every individual has a unique event time $t_i$, which is a positive random variable. By definition, we have $t_i 
\ge \tau_{i}$ almost surely (a.s). We also define $c_i$, which marks the cessation of observation for the individual $i$; this time is referred to as the censored time. For instance, this might be applicable in cases where the observation period has a predetermined ending. In a more general context, $c_i$ can be a positive random variable satisfying $c_i \ge  \tau_{i}$ a.s. Given that some individuals are censored before the event occurs, and vice versa, it is natural to define the observed time as $u_i := \min\{t_i, c_i\}$. We also define the event indicator $\delta_i :=  \mathds{1}\{ t_i \leq c_i\}$, which provides a way to discern whether an event has happened or if it is censored. For each individual $i \in \{1, \ldots, N\}$, we obseve the random variables $(u_i, \delta_i) \in \mathbb{R}_+ \times \{0,1\}$. Furthermore, we suppose that both $t_i$ and $c_i$ are independent across all individuals.

Explanatory variables are defined to give context through time to each of the individuals, and these will be represented by left continuous functions $x_i: \mathbb{R_+} \rightarrow \mathbb{R}^d$. The explanatory variables $x_i(t) \in \mathbb{R}^d$ combine covariates of the individual $i \in \{1, \ldots, N\}$ at time $t \ge 0$. It's important to note that we use the variable $t$ to refer to time in general, while $t_i$ represents the specific event time of individual $i$. We assume that given $x_i$, a short notation for $(x_i(t))_{t \ge 0}$, the times $t_i$ and $c_i$ are conditionally independent. Additionally, we suppose $t_i$ follows a continuous distribution of density $f(t | x_i, \tau_i)$ and $c_i$ a continuous distribution of density $g(t | x_i, \tau_i)$. We have $g(t | x_i, \tau_i) = f(t | x_i, \tau_i) = 0$ for all $t < \tau_i$ since $t_i, c_i \ge \tau_i$ a.s.

In addition, we suppose that $g$ satisfies the following property: \[ 
\forall t \geq \varepsilon > 0: g(t\vert x_i,\tau_i) =  g(t-\varepsilon \vert x_i, \tau_i - \varepsilon)\, .  
\]
Note that this assumption is also necessary for the density function $f$. However, as we will know its specific shape, the property is inherently satisfied. Finally, we denote by $I_d$ the identity matrix of dimension $d$.

\subsection{Survival probability} 
The objective of survival analysis is to predict the length of time until a specified event occurs. Consequently, it is necessary to estimate the distribution of these events. We define the survival probability function of individual $i$ to be the complement of the cumulative distribution, that is, 
$S(t | x_i, \tau_i) = 1 - \int_{\tau_i}^{t} f(s | x_i, \tau_i) ds$, which can also be expressed as the probability of surviving up to time $t$:
\[S(t|x_i,\tau_i) = \mathbb{P}(t_i \geq t | x_i,\tau_i)\,,\qquad t\ge 0\,. \] 
To estimate this function, it is common to assume a particular shape for the hazard function. The hazard function is defined as:
\[ H(t|x_i,\tau_i) = - \frac{\partial}{\partial t} \log(S(t|x_i,\tau_i))\,, \qquad t\ge 0\,,\] 
which represents the instantaneous risk of the event occurring at time $t$. Notably, we can derive the survival function from the hazard function:
\[ S(t|x_i,\tau_i) = \exp \left( - \int_0^t H(s|x_i,\tau_i) ds\right)\,,\qquad t\ge 0\,.\]
For more details on event times distributions, refer to Cox and Oakes \cite{cox1984analysis}.

\subsection{Likelihood}
In order to estimate the survival probability we suppose the hazard function is a function of a specified parametric family $\Theta$ given the explanatory variables. The parameters will be determined following the likelihood principle observing $(u_i, \delta_i) \in \mathbb{R}_+ \times \{0,1\}$ and knowing $(x_i,\tau_i)$. As usual we implicitly make the assumption of non-informative censoring (see Kalbfleisch et al. \cite{kalbfleisch2011statistical}), which means that the censored distribution does not involve the parameter $\theta$.

As mentioned earlier, some models assume a specific shape for the hazard function, such as additive, exponential, logistic or Weibull (see Cox and Oakes \cite{cox1984analysis}). In this paper, we assume that the hazard function is exponential, and we detail this assumption below.

\begin{Definition}[Log-linear regression model for the Hazard function] \label{ass:exp}
    We assume that there exist a vector $\theta \in \mathbb{R}^d$ , such that the hazard function satisfies for all $t\geq 0$ and all $x_i:\mathbb{R}_+ \to \mathbb{R}^d$,
    \[
        H(t|x_i,\tau_i) := h\big( \theta^Tx_i(t)\big)\mathds{1}\{t\ge \tau_i\} \,,\qquad t\ge 0\,,
    \]
    where $h:x \in \mathbb{R} \mapsto \exp(x)$ is the response function.
\end{Definition}

By using this exponential model we obtain a formula to compute the negative log-likelihood which is the function that we aim to minimize.

\begin{proposition}\label{likelihood} Under the exponential model from Definition~\ref{ass:exp} and omitting additional constants, the negative log-likelihood function $\ell :\Theta \rightarrow \mathbb R$ can be written in the following way: \begin{equation}\label{eq:likelihood} \ell(\theta)  = \displaystyle \sum_{i=1}^N -\delta_i \theta^T x_i(u_i) + \int_{\tau_i}^{u_i}\exp(\theta^T x_i(s))ds\,.\end{equation}
\end{proposition}

We call this function the complete log-likelihood and the proof of this proposition is detailed in Appendix~\ref{AppA}.

\subsection{Sequential likelihood optimization}
We consider a horizon time $n$ and a time partition $(t-1,t]$ with discrete time $t=1,2,\ldots$ that is independent of the observations $(u_i,\delta_i)_{1\le i\le N}$. In many real-life situations, data continues to evolve; new patients may arrive, some patients may leave, and the optimization algorithm may need to update its estimation as new information becomes available. This is the focus of our work: to update online convex optimization algorithms for sequential survival data. 

For individual $i$ we define $y_{it} := \delta_i \mathds{1}\{t-1 < u_i \leq t \}$ which indicates whether an event is observed for individual $i$ during the interval $(t-1,t]$ or not. Additionally, we denote the risk indicator as $r_{it} := \mathds{1}\{\tau_i \leq t, u_i > t-1\}$ for event $i$ in the interval $(t-1,t]$. Then, we define the log-likelihood on the interval $(t-1,t]$ by the expression
\begin{equation}\label{eq:lik} \ell_t(\theta) := \displaystyle \sum_{i=1}^N -y_{it} \theta^T x_i(u_i) + r_{it} \int_{\tau_i \vee (t-1)}^{u_i \wedge t}\exp(\theta^T x_i(s))ds\,,\qquad\theta\in \Theta, ~t=1,2,\ldots \, , \end{equation}
where we remind $u_i \wedge t = \min\{u_i,t\}$ and $\tau_i \vee (t-1) = \max\{\tau_i,t-1\}$. Let us notice that, analogous to Equation \eqref{eq:likelihood}, the contribution to the log-likelihood of an individual that experiences an event in the interval $(t-1,t]$ is given by $\theta^T x_i(u_i) + \int_{\tau_i \vee (t-1)}^{u_i} \exp(\theta^T x_i(s))ds$, and the contribution of an individual that is censored  in the interval $(t-1,t]$—either by $u_i = c_i$ or by $t$—is $ \int_{\tau_i \vee (t-1)}^{u_i \wedge t}\exp(\theta^T x_i(s))ds$. If an individual is not yet present in the interval, i.e., $\tau_i > t$, or its observed time has passed before the beginning of the interval ($u_i \leq t-1$), its contribution to the log-likelihood is zero.

Finally, the log-likelihood up to time $n$ is given by:
\[\ell^n(\theta) := \displaystyle \sum_{t=1}^n \ell_t(\theta),\qquad\theta\in \Theta. \]
It is important to notice that if $n$ is sufficiently large, i.e., $n\ge u_i$ for every $1\le i\le N$, and all the events have been observed, the complete log-likelihood of Equation \eqref{eq:likelihood} corresponds to the sum of all the interval contributions. Therefore, $\ell(\theta) = \ell^n(\theta)$, $\theta\in \Theta$, for $n$ sufficiently large when $N$ is finite.

\section{Online convex optimization}
\label{sec:OCO}
\subsection{Setting} 
\label{subsec:4.1}
A convex optimization problem consists of approximating the minimum of a convex function over a convex set. This problem can be extended to a recursive setting where, at each iteration $t$, a convex optimization algorithm predicts the parameter $\theta_t$ and incurs a loss of $\ell_{t}(\theta_t)$. This approach is particularly good in situations where the data evolves over time, requiring fast adaptation and decision making. We apply this methodology to survival analysis, introducing a novel perspective in a field traditionally dominated by batch processed data.

The online convex optimization algorithm aims to minimize its regret at any horizon time $n\ge1$:
\[Regret_n := \displaystyle \sum_{t=1}^n \ell_{t}(\theta_t) - \min_{\theta \in \Theta} \sum_{t=1}^n \ell_{t}(\theta). \] 
In this paper, we aim to optimize the losses $\ell_t(\theta)$ from Equation \eqref{eq:lik}. To apply online convex optimization algorithms, we must first assume that $\Theta \subseteq \mathbb{R}^d$ is a non-empty, convex, bounded, and closed set. Subsequently, we verify the convexity of the objective function. Here the choice of the response function $h$ is crucial. For $h(x)=\exp(x)$ the cost function $\ell_t (\theta)$ is defined in Equation \eqref{eq:lik} for every iteration $t$. We derive its gradient and Hessian:
\begin{equation} \label{eq:4} \nabla \ell_t(\theta) =  \displaystyle \sum_{i=1}^N -y_{it} x_i(u_i) + r_{it} \int_{\tau_i \vee t-1}^{u_i \wedge t}\exp(\theta^T x_i(s)) x_i(s)ds,\end{equation}
and   
\begin{equation} \label{eq:5} \nabla^2 \ell_t(\theta) =  \displaystyle  \sum_{i=1}^N  r_{it} \int_{\tau_i \vee t-1}^{u_i \wedge t} \exp(\theta^T x_i(s))x_i(s) x_i(s)^Tds \succcurlyeq 0.\end{equation} 
The positive semi-definite Hessian confirms the convexity of the losses. Additionally, we formalize the boundedness assumption. 
\begin{assumption}[Bounded domain and gradient]
    \label{ass:bounded}
    There exists $D, G >0$ such that for all $t=1,2,\ldots$ and $\theta \in \Theta$, $\|\theta\|\leq D$ and $\|\nabla \ell_t(\theta)\| \leq G$. 
\end{assumption}
One of the most ancient algorithms for online convex optimization is named "follow the leader" (FTL), and it consists of choosing, at each iteration $t$, the point that optimizes the cumulative loss up to $t-1$. This algorithm does not satisfy any non-trivial regret guarantee for linear losses. However, under some modifications, like the randomized version proposed by Hannan \cite{hannan1957approximation}, it can achieve an $\mathcal{O}(\sqrt{n})$ regret bound. Additionally, the approach from Cesa-Bianchi and Lugosi \cite{cesa2006prediction}, where the losses are strongly convex, achieves a logarithmic regret in the number of iterations.

In 2003, Zinkevich \cite{zinkevich2003online} proposed a sequential version of the gradient descent algorithm (OGD), which satisfies a uniform regret bound of $\mathcal{O}(\sqrt{n})$ for an arbitrary sequence of convex cost functions and under the previous conditions (bounded gradients and domain). Later, Hazan et al. \cite{hazan2007logarithmic} proved that Zinkevich's algorithm attains a $\mathcal{O}(\log(n))$ regret for an arbitrary sequence of strongly convex functions (with bounded first and second derivatives). They also introduced an online version of the Newton-Raphson method, which they named the Online Newton Step (ONS), and demonstrated that it also achieves logarithmic regret. More algorithms and details can be found in Hazan \cite{hazan2022introduction}.

We implement the ONS algorithm to minimize the negative log-likelihood and study the selection of its hyperparameters along with its regret bounds.

\subsection{Exp-concavity and directional derivative condition}
To ensure a logarithmic regret bound, the loss function must satisfy specific conditions. First, we review the definition of exp-concavity.
\begin{Definition}(Exp-concavity)
A convex function $\ell: \Theta \rightarrow \mathbb{R}$ is $\mu$-exp-concave iff the function $p(\theta):=\exp(-\mu \ell(\theta))$ is concave.
\end{Definition}
This property is fundamental in the regret analysis and replaces the strong convexity condition required by the OGD algorithm. This means that the ONS algorithm requires a weaker hypothesis on the losses $(\ell_t)_{t=1,2,\ldots}$, to achieve logarithmic regret. Furthermore, we introduce a study based on this weaker condition, which is essential to derive the regret bound described by Hazan \cite{hazan2022introduction} in survival analysis.

\begin{Definition}(Directional derivative condition -- DDC) We say a function $\ell:\Theta \rightarrow \mathbb{R}$ satisfy the directional derivative condition for a constant $\gamma >0$ if  for any pair $\theta_1, \theta_2 \in \Theta$ \begin{equation} \ell(\theta_2) \geq \ell(\theta_1) + \nabla \ell(\theta_1) (\theta_2-\theta_1) + \frac{\gamma}{2}\left(\nabla \ell(\theta_1) (\theta_2-\theta_1) \right)^2. 
\label{eq:DDC}
\tag{DDC}
\end{equation}
\end{Definition}

To determine the directional derivative constant $\gamma$, we must first compute the exp-concavity constant $\mu$.

\begin{Lemma}
\label{lem:1}
A twice differentiable function $\ell:\Theta \rightarrow \mathbb{R}$ is $\mu$-exp-concave iff \begin{equation} \label{eq:lem1} \nabla^2 \ell(\theta) \succcurlyeq \mu \nabla \ell(\theta) \nabla \ell(\theta)^T, \qquad \theta \in \Theta. \end{equation}
This holds with 
\[
\mu \le \min_{\theta\in \Theta } \frac{ \nabla \ell(\theta)^T \nabla^2 \ell(\theta) \nabla \ell(\theta)}{||\nabla \ell(\theta)||^4}.
\]
\end{Lemma}

This lemma provides us a way for calculating the exp-concavity constant $\mu$. The proof of Lemma~\ref{lem:1} can be found in Appendix~\ref{AppB}. Given a $\mu$-exp-concave function $\ell$, we can also determine its directional derivative constant $\gamma$. We have the following bound:

\begin{Lemma} \label{lem:2}
A $\mu$-exp-concave function $\ell : {\Theta}\to \mathbb R$, satisfying Assumption~\ref{ass:bounded}, admits a directional derivative constant $\gamma>0$ satisfying
\[
 \gamma \le \min_{\theta\in\Theta } \frac{- \frac{2}{\mu}\log(1 + \mu ||\nabla \ell(\theta) || D) + ||\nabla \ell(\theta) || D}{(|| \nabla \ell(\theta) || D)^2}\,.
\]
\end{Lemma}

We note that this lower bound improves upon the upper bound provided by Hazan \cite{hazan2022introduction}:
\[
\gamma\le \dfrac12 \min \Big\{\dfrac1{GD},\mu\Big\}\,,
\]
and the proof of Lemma~\ref{lem:2} can also be found in Appendix~\ref{AppB}.

\subsection{Online Newton Step}
The ONS algorithm is an online analogue of the Newton-Raphson method; see Ypma \cite{ypma1995historical} for more details. The Newton-Raphson algorithm moves in the direction of the inverse of the Hessian multiplied by the gradient. For exp-concave loss functions $\ell_t$ with $t=1,2,\ldots$, we can replace the Hessian matrix with an approximation of it:
\[A_t = \displaystyle  \sum_{k=1}^t \nabla \ell_{k}(\theta_k) \nabla \ell_{k}(\theta_k)^T.\]
At each iteration, the algorithm updates the estimation of the parameter as follows:
\[\theta_{t+1} = \theta_t - \frac{1}{\gamma} A_t^{-1} \nabla \ell_{t} (\theta_t)\,, \]
where $\gamma$ is an algorithm hyperparameter denoting the learning rate and its optimal selection aligns with the DDC constant. This might lead to a point outside the convex set $\Theta$ and so we need to project it back. This projection is somewhat different than the standard projection as it is characterized by the norm defined by $A_t$ instead of the Euclidean norm. The iteration step of the algorithm is:  
\[\theta_{t+1} = {\rm Proj}_t\Big(\theta_t - \frac{1}{\gamma} A_t^{-1} \nabla \ell_{t}(\theta_t)\Big)\,, \]
where ${\rm Proj}_t(\theta^*) \in  \displaystyle \arg\min_{\theta\in \Theta}(\theta - \theta^*)^T A_t (\theta -\theta^*)$. 

Let us remark that ONS requires to invert a large matrix $A_t$, and in order to avoid expensive calculations, we consider the Sherman-Morrisson formula \cite{sherman1950adjustment} which provides a recursion for $A_t^{-1}$ from $A^{-1}_0 := (1/\epsilon) I_d$:
\[A_t^{-1} = A^{-1}_{t-1} - \frac{A_{t-1}^{-1} \nabla \ell_{t}(\theta_t) \nabla \ell_{t}(\theta_t)^T A_{t-1}^{-1}}{1+ \nabla \ell_{t}(\theta_t) A_{t-1}^{-1} \nabla \ell_{t}(\theta_t)^T}\,,\qquad t=1,2,\ldots.\]
We formally describe the Online Newton Step algorithm~\ref{alg:ONS} in Appendix~\ref{AppB}. Hazan \cite{hazan2022introduction} proved the following regret bound of ONS.
\begin{Theorem}[Hazan \cite{hazan2022introduction}]
\label{thr:1}
    Let us consider the losses $\ell_t:\Theta \rightarrow \mathbb R$ $\mu$-exp-concave and satisfying Assumption~\ref{ass:bounded}. Then, Algorithm~\ref{alg:ONS} with hyperparameters $\gamma = \frac{1}{2} \min \{\frac{1}{GD},\mu \}$ and $\epsilon = (\gamma D)^{-2}$ satisfies
    $Regret_n \leq  \gamma^{-1} d \log(2nG^2\gamma^2D^2)$ for any $n\geq 4 $.
\end{Theorem}
Let us remind that we want to apply ONS algorithm to the losses $(\ell_t)_{t=1,2,\ldots}$ described in Equation \eqref{eq:lik}, where we assume the exponential model defined in~\ref{ass:exp}. The exp-concavity property is fundamental in the regret analysis of ONS. We first notice that we can work under (DDC) rather than $\mu$-exp-concavity, focusing our work on the study of the constant $\gamma$ which is the hyperparameter of ONS, rather than on $\mu$. Then we see that the choice of this constant is very sensitive to variations in the gradients, which depend on the number of people at risk at each time. If $\mu$ is small, which can happen when the gradient of the loss is small, the choice of  $\gamma = \frac{1}{2}\min\{\frac{1}{GD},\mu\}$ proposed by Hazan \cite{hazan2022introduction} will also be small, potentially exploding the regret bound and causing issues with the convergence. Additionally, to properly tune the hyperparameter $\gamma$ we need to know the exp-concavity constant in advance, but this constant might depend on the gradient of the losses that are not known before running the algorithm. Adjusting $\gamma$ is not trivial and we provide some insights in the following sections.


\section{Stochastic setting}
\label{sthsection}
The first solution we propose is to use a stochastic approach to bound the regret of ONS. We present the general stochastic setting introduced by Wintenberger~\cite{wintenberger2021stochastic} and apply one of its results to the survival case. The main difficulty in sequential survival analysis is the intrinsic time dependence in the loss functions $(\ell_t)_{t=1,2,\ldots}$. Indeed, even if the individuals are iid, the log-likelihoods $\ell_t$ are dependent because of the individuals that are at risk during consecutive time intervals $(t-1,t]$ for $t=1,2,\ldots$.

\subsection{Stochastic Model}
We model the arrival times $\tau_i \ge 0$ as a homogeneous Poisson process with intensity $\lambda$; see Kingman \cite{kingman1992poisson} for a reference textbook on the subject. For each $t> 0$, we define the count random variable $N_t :=  \sum_{i=1}^{\infty} \mathds{1}\{  \tau_i \leq t\}$ which represents the number of individuals that arrive before $t$, and $\tau_{N_t}$ represents the arrival time of the last individual arriving before $t$. We assume a constant rate $\lambda$, such that $\mathbb{E}[N_t] = \lambda t$, indicating the average number of individuals arriving at time $t$. Additionally, in this section, we consider the covariate functions to be constant, i.e., $x_i(t) = x_i$ for all $t >0$, and that they follow, independently, the distribution of a random variable $X$. In this stochastic setting we rewrite the loss function: \begin{equation}\label{stcloss}\ell_t(\theta) = \displaystyle \sum_{i=1}^{N_t} -y_{it} \theta^T x_i + r_{it} \exp(\theta^T x_i)\big( (u_i \wedge t) - (\tau_i \vee (t-1))\big), \end{equation} 
where we replaced $N$ by $N_t$ in Equation \eqref{eq:lik}. It is important to note that the derivation of this expression is based on the assumption of the exponential model~\ref{ass:exp}. Now, we want to apply ONS to optimize this loss and study what happens with its stochastic regret. 

For each iteration $t=1,2,\ldots$, we consider the stochastic loss $\ell_t$ and the filtration $\mathcal{F}_t$ of $\sigma$-algebras such that the predictions of the online learning algorithm $\theta_t$ and the past losses $(\ell_s)_{s=1}^{t-1}$ are $\mathcal{F}_{t-1}$-measurable. To simplify notation, we use $\mathbb E_t[\cdot]$ to represent the conditional expectation given $\mathcal F_t$, denoted as $\mathbb E[ \cdot \vert \mathcal F_t]$. In this context, our objective is to minimize the stochastic regret at any horizon time $n \ge 1$:
\[
Risk_n := \displaystyle \sum_{t=1}^n L_t(\theta_t) - \min_{\theta \in \Theta} \sum_{t=1}^n L_t(\theta)\,,
\] 
where $L_t(\theta_t)$ is the conditional risk, defined as $L_t(\theta_t) := \mathbb{E}_{t-1}[\ell_{t}(\theta_t)]$ for $t=1,2,\ldots $. Let us notice that in our case, where the stochastic losses $\ell_t$ are defined in Equation \eqref{stcloss}, the $\sigma$-algebra $\mathcal F_t$ is generated by $y_{is}$, $x_i$, $\tau_i$, and $u_{is} = \min\{u_i, s\}$ for all $i = 1, \ldots, N_{t-1}$ and $s = 1, \ldots, t-1$.

The main difference with the setting presented in Section~\ref{sec:OCO} is the use of the conditional risk $L_t$ instead of the loss functions $\ell_t$ in the calculation of regret. This allows us to relax the convexity conditions imposed on $\ell_t$ and instead focus on the convexity properties of~$L_t$.

\subsection{Stochastically Exp-Concavity}
It was proved in Wintenberger~\cite{wintenberger2021stochastic} that the ONS algorithm achieves a $\mathcal{O}(\log(n))$ stochastic regret bound under a stochastic exp-concavity condition for $\ell_{t}$ which is described below.

\begin{Definition}[Stochastic exp-concavity]
\label{con:1} A sequence of random functions $(\ell_{t})_{t=1,2,\ldots}$ is said to be $\gamma$ stochastically exp-concave with respect to a filtration $\mathcal F_t$ if for all $\theta_1,\theta_2 \in \Theta$ and $t=1,2,\ldots$  
\begin{equation*} L_t(\theta_1) \leq L_t(\theta_2) + \nabla L_t(\theta_1)^T (\theta_1-\theta_2) - \frac{\gamma}{2} \mathbb{E}_{t-1} \left[(\nabla \ell_{t}(\theta_1)^T(\theta_1-\theta_2))^2 \right], \qquad a.s.\end{equation*}
\end{Definition}

Let us note that this property corresponds to the stochastic counterpart of the directional derivative condition~\eqref{eq:DDC}. This property plays a crucial role in the proof of Theorem 7 of Wintenberger \cite{wintenberger2021stochastic}, which establishes the logarithm stochastic regret bound. However, the losses $\ell_t$ defined in \eqref{stcloss} do not satisfy this property. Nevertheless, we demonstrate that the events where this inequality is not fulfilled have a small probability and therefore, we can still bound the stochastic regret. In addition, we need to make the following design assumption.

\begin{assumption}\label{ass:strongconv}
There exist $A>0$ such that 
$\mathbb E [xx^{\top} \mathds 1\{ T \leq C \} (1-T)_+ \vert \tau = 0 ] \succcurlyeq A I_d.$
\end{assumption}

This assumption is not trivial, and it is not always satisfied; however, when all the individuals experience an event and $T\leq 1$, it corresponds to a classical design. When $t \ge 1$ an alternative analyses is required.

\subsection{Stochastic Regret}
To apply Theorem 7 from Wintenberger \cite{wintenberger2021stochastic}, the losses need to satisfy certain hypothesis, among which are stochastic exp-concavity and a stochastic bound on the gradients of the losses. We prove that our losses $\ell_t$, whose do not satisfy exactly the conditions of Theorem 7, still leads ONS algorithm to achieve a logarithmic stochastic regret. We present the result in the following theorem.

\begin{Theorem}\label{thr:stcregret}
Given $\varrho >0, n \ge1$ and the stochastic losses $(\ell_t)_{t=1,2,\ldots}$ from Equation \eqref{stcloss}, then under Assumption~\ref{ass:strongconv}, a bounded domain of diameter $D$ and hyperparameter $\gamma$, the stochastic exp-concavity constant, the ONS algorithm has logarithmic stochastic regret with probability $1-4\varrho$. Specifically, $Risk_n=\mathcal O(\log(n/\varrho)/\gamma)$, $n\ge 1$.
\end{Theorem}

The proof of Theorem~\ref{thr:stcregret} can be found in Appendix~\ref{AppC} and the explicit regret bound in Equation \eqref{stc:explicitbound}. To finish, we prove the following corollary.

\begin{Corollary}\label{cor:stc}
Given $\varrho >0, n\ge1$ and the stochastic losses $(\ell_t)_{t=1,2,\ldots}$ from Equation \eqref{stcloss}, we consider $\theta_t$ the ONS prediction at time $t$ and $\bar \theta_n$ the average prediction $\bar \theta_n = \frac{1}{n} \sum_{t=1}^n \theta_t$. Defining the optimal parameter
\[\theta^* = \arg\min_{\theta \in \Theta}  \sum_{t=1}^n L_{t}(\theta),\]
then, under Assumption~\ref{ass:strongconv}, a bounded domain of diameter $D$ and hyperparameter $\gamma$, with probability $1-4\varrho$ we have:
\[
 ||\bar\theta_n - \theta^* ||^2 \leq \mathcal O \left(\frac{\log(n/\varrho)}{\gamma n}\right)\,, n\ge 1 .
\]    
\end{Corollary}

This corollary ensures the convergence of the algorithm predictions to the real parameter, which is possible thanks to the strong convexity of the risk functions $L_t$. It is important to remark that this does not hold in the deterministic setting. The proof can be found in Appendix~\ref{AppC} and the explicit bound in Equation \eqref{cor:explicitbound}.

\section{Survival ONS algorithm}
As mentioned earlier, the choice of $\gamma$ has a significant influence on the algorithm's performance, particularly regarding the regret bound. To avoid convergence issues and address the challenge posed by the small optimal constant proposed by Hazan \cite{hazan2022introduction}, we propose an adaptive setting that allows us to select the most suitable learning rate at each step while maintaining control over the regret bound. We introduce SurvONS (Algorithm~\ref{alg:BOAm}), a survival version of MetaGrad from van Erven et al. \cite{van2021metagrad}, that uses Bernstein Online Aggregation (BOA, introduced in Wintenberger \cite{wintenberger2017optimal}) to aggregate multiple ONS applied to an adaptive auxiliary function. SurvONS strategically selects larger learning rates to handle sub-optimal parameters. The key difference between our algorithm and MetaGrad lies in the approach to updating the adaptive learning rate. We explain this algorithm in detail throughout this section.
 
\subsection{Recursive adaptation to the constants}
We present first the recursive adaptation of the constants $\mu$ and $\gamma$. Lemma~\ref{lem:1} provides a bound for the exp-concavity constant $\mu$, and Lemma~\ref{lem:2} offers a bound for the directional derivative constant $\gamma$ based on $\mu$. We aim to apply this approach to $\ell = \ell_t$ for all $t = 1,2,\ldots$, and recursively obtain $\mu_t$ and $\gamma_t(\mu_t)$ such that they satisfy the bounds of Lemma~\ref{lem:1} and Lemma~\ref{lem:2}.

In Hazan's approach, as described in \cite{hazan2022introduction}, the idea is to select a universal constant $\mu$ that renders all the functions $(\ell_t)_{t=1,2,\ldots,}$ $\mu$-exp-concave. The natural choice would be to take: \[\mu := \displaystyle \min_{t \in \{1,\ldots,n\}} \mu_t^*,\qquad \text{where} \qquad \displaystyle \mu_t^* := \min_{\theta \in \Theta} \frac{\nabla \ell_t(\theta)^{\top} \nabla^2\ell_t(\theta)\nabla \ell_t(\theta)}{ || \nabla \ell_t(\theta)||^4}, \]

is the bound given by Lemma~\ref{lem:1}. With this configuration, we guarantee exp-concavity for every function. However, the challenge of minimizing over the parameter set $\Theta$ in the definition of $\mu_t^*$ might be more intricate than minimizing the loss function $\ell_t$. In addition, we can not know the constant in advance because $\ell_t$ is revealed at the $t$-th iteration only in our online setting. 

To solve this problem, we define at each time $t=1,2,\ldots$ an adaptive estimation of the exp-concavity constant:

\[\mu_t := \frac{\nabla \ell_t(\theta_t)^{\top} \nabla^2\ell_t(\theta_t)\nabla \ell_t(\theta_t)}{ || \nabla \ell_t(\theta_t)||^4}, \] and, similarly from Lemma~\ref{lem:2},  \[ \gamma_t (\mu_t) :=  \frac{- \frac{2}{\mu_t}\log(1 + \mu_t ||\nabla \ell_t(\theta_t) || D) + 2 ||\nabla \ell_t(\theta_t) || D}{(|| \nabla \ell_t(\theta_t) || D)^2}, \] 

where $\theta_t$ is the parameter predicted by the algorithm at time $t$. Let us notice that this choice of $\mu_t \geq \mu$ and $\gamma_t (\mu) \geq \gamma(\mu)$ assures the exp-concavity and the directional derivative condition for $\ell_t$ close to $\theta_t$ at time $t$.
We sometimes refer to $\gamma_t(\mu_t)$ as $\gamma_t$ when the specification is not necessary. 

\subsection{SurvONS}
Now, we have an adaptive way to choose $\mu_t$ and $\gamma_t$ that preserves the exp-concavity properties at each iteration. However, this choice might not be optimal, in some iterations the gradient $\nabla \ell_t$ can be close to zero due to the lack of individuals at risk, and this might lead to numeric problems setting $\mu_t$ and $\gamma_t$. Thus we propose an intermediate choice of the learning rate. Given a user specified constant $\gamma >0$, we define for each time $t =1,2\ldots$: 
\[  \tilde \gamma_t := \max\{\gamma_t(\mu_t)/4,\gamma \}\, ,\] 
a value that chooses a portion of the optimal directional derivative condition constant $\gamma_t/4$ when it is not too small, and the user specified constant $\gamma$ when the quarter of the optimal constant decreases under $\gamma$. This choice $\tilde\gamma_t$ is a trade off in between choosing the optimal directional derivative condition constant and a worse constant when the optimal one is susceptible to bring convergence problems.

In order to keep the logarithmic regret bound we cannot directly use the adaptive choice of the constant as the algorithm's learning rate. Instead, it was proposed by van Erven et al. \cite{van2021metagrad} to optimize an adaptive auxiliary function. Let us consider $\hat\theta$ such that $\ell_t(\hat\theta)$ and $\nabla \ell_t(\hat\theta)$ have been observed and $\gamma >0$, we define the directional derivative function: \begin{equation} \label{eq:auxfunc} \hat{\ell}_{t,\gamma}(\theta) := \ell_t(\hat\theta) + \nabla \ell_t(\hat\theta)(\theta - \hat\theta) + \frac{\gamma}{2} \left(\nabla\ell_t(\hat\theta)(\theta - \hat\theta)  \right)^2, \qquad \theta\in \Theta\,,\qquad t=1,2,\ldots\,.
\end{equation} We prove that this function satisfies the directional derivative condition for a different constant $\hat{\gamma}$. 

\begin{Lemma}\label{lem:aux} Let $\gamma >0$, $\Theta \subseteq \mathbb{R}^d$ of diameter $D>0$, $\hat \theta \in \Theta$ and $\ell_t:\Theta \rightarrow \mathbb{R}$ the log-likelihood defined in Equation \eqref{eq:lik}. Then, the function $\hat{\ell}_{t,\gamma}$ from \eqref{eq:auxfunc} satisfies for every $\theta_1,\theta_2 \in \Theta$: \begin{align*} \hat{\ell}_{t,\gamma}(\theta_2) & \geq \hat{\ell}_{t,\gamma}(\theta_1) + \nabla \hat{\ell}_{t,\gamma}(\theta_1)(\theta_2-\theta_1)\\ & \hspace{2cm} + \frac{\gamma}{2(1+\tilde\gamma\nabla \ell_t(\hat\theta)(\theta_1-\hat\theta))^2}\left( \nabla \hat{\ell}_{t,\gamma}(\theta_1)(\theta_2-\theta_1)\right)^2, 
\end{align*} 
and thus, the function $\hat{\ell}_{t,\gamma}$ has directional derivative constant $\hat{\gamma}$ with $\hat{\gamma} := \frac{\gamma}{2(1+\gamma D ||\ell_t (\hat\theta)|| )^2}$.
\end{Lemma}

The proof of Lemma~\ref{lem:aux} is presented in Appendix~\ref{AppD}. The idea of the algorithm is to use ONS routine to optimize the functions $\hat{\ell}_{t,\gamma} = \hat{\ell}_{t,\tilde \gamma_t}$, i.e., the auxiliary function with $\gamma = \tilde \gamma_t$, which adapt at each step according to the current optimal $\gamma_t$ and the algorithm predictions $\theta_t$.

In addition, to obtain an algorithm that is robust for the choice of the learning rate, we propose an aggregation procedure which applies ONS and combines it with multiple choices of the learning rate $\gamma$. To formalize this idea, we consider a grid $\Gamma=\{\gamma_i\}_{i=1,\ldots,K}$ and $\mathcal{E} = \{ \epsilon_i\}_{i=1,\ldots,K}$ such that  $\epsilon_i = \frac{1}{(\gamma_i D)^2}$ for all $i=1,\cdots,K$. Then, at each iteration $t = 1,\ldots,n$, and for each $i = 1,\ldots,K$, we define $\tilde\gamma_{it} =  \max\{ \gamma_t/4,\gamma_i\}$ and we aggregate ONS applied to $(\hat{\ell}_{t, \tilde\gamma_{it}})_{t=1,2,\ldots}$. The aggregation is held by BOA algorithm of Wintenberger \cite{wintenberger2017optimal}, which is a recursive procedure that considers exponential weights with a second order refinement. The algorithm SurvONS is described in Algorithm~\ref{alg:BOAm} and it is important to notice that the difference between SurvONS and MetaGrad is the choice of the constant~$\tilde \gamma$.

\begin{algorithm}[t]
\caption{SurvONS}
    \begin{algorithmic}
    \STATE {\bfseries Input: }$(\ell_t)_{t=1,2,\ldots}, D>0, n\geq1,$ grids $\Gamma, \mathcal{E}$\\
    \STATE {\bfseries Initialization:} for each $\gamma_k$ in $\Gamma$:   $\theta_0(\gamma_k) \in \Theta$, $\pi_{0,k} = \frac{1}{K}$, $\hat{\theta}_0 \in \Theta $, $A^{-1}_0 = \mathcal{E}^{-1} \mathds{1}_d$\\
    \FOR{iteration $t=1,\ldots,n$}
        \STATE \textbf{Update:} $\hat{\theta}_t = \displaystyle \sum_{k=1}^K \pi_{t,k} \, \theta_t(\gamma_k) $
        \STATE \textbf{Observe:} $\nabla\ell_t(\hat{\theta}_t) =  \displaystyle \sum_{i=1}^N -y_{it} x_i(u_i) + r_{it} \int_{\tau_i \vee t-1}^{u_i \wedge t}\exp(\hat{\theta}_t^T x_i(s)) x_i(s)ds $\\ ~~~~~~~~~~~~~ $\nabla^2\ell_t(\hat{\theta}_t) =  \displaystyle  \sum_{i=1}^N  r_{it} \int_{\tau_i \vee t-1}^{u_i \wedge t} \exp(\hat{\theta}_t
^T x_i(s))x_i(s) x_i(s)^Tds $\\
         $~~~~~~~~~~~~~~ \mu_t = \frac{\nabla\ell_t(\hat{\theta_t})^T \nabla^2\ell_t(\hat{\theta_t})\nabla\ell_t(\hat{\theta_t})}{||\nabla\ell_t(\hat{\theta}_t)||^4}$\\  $~~~~~~~~~~~~~~ \gamma_t = 2\frac{- \frac{1}{\mu_t}\log(1 + \mu_t || \nabla\ell_t(\hat{\theta_t}) || D) + || \nabla\ell_t(\hat{\theta_t}) || D}{(|| \nabla\ell_t(\hat{\theta_t}) || D)^2}$\\
        \FOR{$\gamma_k \in \Gamma$}
        \STATE \textbf{Observe:} $\tilde \gamma_t = \max \{ \gamma_t/4,  \gamma_k \} $\\ ~~~~~~~~~~~~~~$\nabla \hat{\ell}_{t,\tilde \gamma_t}(\theta_t(\gamma_k))= \nabla \ell_t(\hat{\theta}_t)(1+ \tilde\gamma_t\nabla \ell_t(\hat{\theta}_t)(\theta_t(\gamma_k) - \hat{\theta}_t))$
        \STATE \textbf{Recursion:} 
                 
         \STATE $$A^{-1}_t = A^{-1}_{t-1} - \frac{A_{t-1}^{-1} \nabla \hat{\ell}_{t,\tilde \gamma_t}(\theta_t(\gamma_k)) \nabla \hat{\ell}_{t,\tilde \gamma_t}(\theta_t(\gamma_k))^T A_{t-1}^{-1}}{1+ \nabla \hat{\ell}_{t,\tilde \gamma_t}(\theta_t(\gamma_k)) A_{t-1}^{-1} \nabla \hat{\ell}_{t,\tilde \gamma_t}(\theta_t(\gamma_k))^T} $$\\ $$\theta_{t+1}(\gamma_k) = {\rm Proj}_t\left(\theta_t(\gamma_k) - \frac{1}{\gamma_k} A_t^{-1} \nabla  \hat{\ell}_{t,\tilde \gamma_t} (\theta_t(\gamma_k))  \right)$$
        \ENDFOR
        \STATE \textbf{Update:} $\pi_{t+1,\cdot} =  \pi_t \exp\left( -\Gamma \nabla \ell_t(\hat{\theta}_t)^T(\hat{\theta}_t - \theta_t(\Gamma))- \Gamma^2(\nabla \ell_t(\hat{\theta}_t)(\hat{\theta}_t - \theta_t(\Gamma)))^2\right)$
    \ENDFOR
    \RETURN $\hat{\theta}_n$
    \end{algorithmic}
    \label{alg:BOAm}
\end{algorithm}

Aggregation methods allow us to avoid bad choices of $\gamma$ and therefore, the convergency issues. Let us remind that we consider the exponential model~\ref{ass:exp}. We prove that the regret of Algorithm~\ref{alg:BOAm} is bounded.

\begin{Theorem}\label{regbnd} Let $n\ge1$, $(\ell_t)_{ t=1,\ldots,n}$ be the sequence of losses defined in~\eqref{eq:lik}, that are assumed to satisfy Assumption~\ref{ass:bounded} and \eqref{eq:DDC} with constants $\gamma_t \in (0,1/GD)$. Let $K\geq 1$ and $\Gamma \in (0,1/(4GD))^K$. Then, Algorithm~\ref{alg:BOAm} with hyperparameters $\Gamma$ and $\mathcal E = 1/(\Gamma D)^2$, satisfies the regret upper-bound: 
\[ Regret_n \leq  \min_{\gamma \in \Gamma} \bigg\{ \frac{2\log(K)+5d\log(n)}{\gamma} + \gamma G^2D^2n_{\gamma}\bigg\} \, ,\]
where $ n_{\gamma} :=   \sum_{t=1}^n \mathds{1}\{\gamma_t < \gamma\}$, $\gamma>0$.
\end{Theorem} 

This theorem provides a regret bound that proposes a trade-off between the bad choices of $\gamma$ and the frequency with which the algorithm selects $\gamma$ over $\gamma_t$, thereby compensating for the regret increment. The proof of Theorem~\ref{regbnd} can be found in Appendix~\ref{AppD}. Let us notice that this analysis is also valid for MetaGrad algorithm \cite{van2021metagrad}, and Theorem~\ref{regbnd}, which was developed for the survival losses \eqref{eq:lik}, holds equally true for any loss satisfying Assumption~\ref{ass:bounded} and \eqref{eq:DDC}.

\subsection{Theoretical regret bounds comparison}
We show in Figure~\ref{fig:theoreticbound} the differences between the regret bound orders, in order to illustrate the importance of the constant adaptation $\tilde \gamma_t$ in SurvONS, and the interest of the stochastic setting. We compare the theoretical regret bound orders of ONS \cite{hazan2007logarithmic} with the optimal hyperparameter $\gamma_t$, OGD \cite{zinkevich2003online}, SurvONS~\ref{alg:BOAm}, and ONS with an average hyperparameter $\bar\gamma_t = \sum_{s=1}^t \gamma_s$, representing the stochastic approach. The bounds are detailed in Table~\ref{tab:regretbounds}.

\begin{table*}[t]
\centering
\caption{Regret bound order after $n$ iterations (up to logarithmic factors)}
\label{tab:regretbounds}
\begin{tabular}{@{}lccccc@{}}
\hline
 & \multicolumn{1}{c}{OGD}
& \multicolumn{1}{c}{ONS} & \multicolumn{1}{c}{SurvONS}
& \multicolumn{1}{c}{ONS$(\bar\gamma)$} \\
\hline \\[-7pt]
 Regret bound &$\sqrt{n}$ & $\frac{1}{\min_{1 \leq t\leq n} \gamma_t} $ & $\min_{\gamma}\big\{\frac{1}{\gamma} + n_{\gamma}\big\}$ & $\frac{1}{\bar\gamma_n}$ \\
\hline
\end{tabular}
\end{table*}

In this comparison, we omit constants and logarithmic terms. We estimate $\gamma_t$ with SurvONS, and we use these estimations to construct the bounds. The simulation framework for this experiment is detailed in Section~\ref{sec:sim}. The graph is presented in log-log scale.


\begin{figure}[t]
\centering
    \includegraphics[scale=.25]{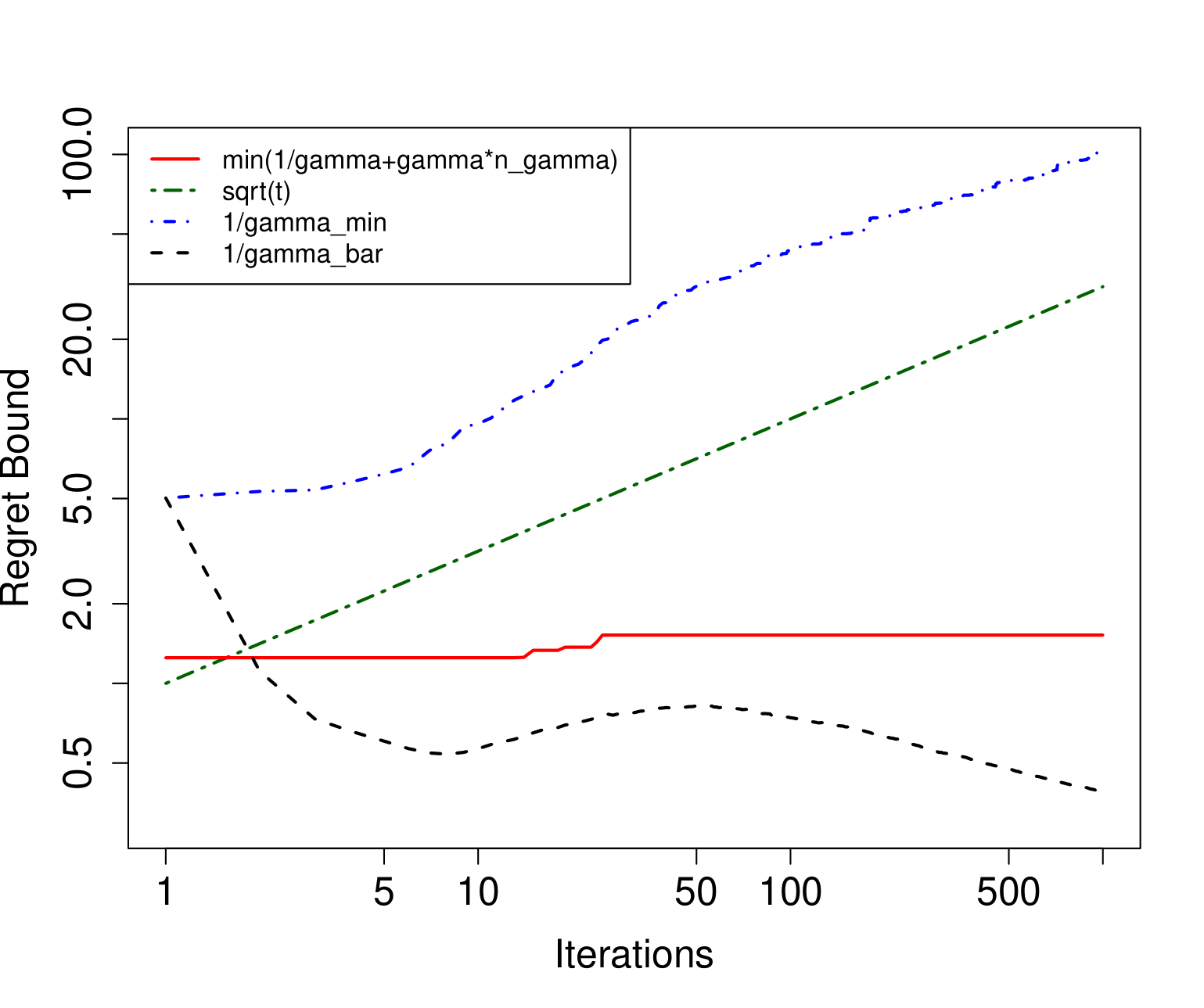}
    \caption{Regret bound orders (up to logarithmic factors)}
    \label{fig:theoreticbound}
\end{figure}

Figure~\ref{fig:theoreticbound} traces the regret behavior of the different algorithms (see Table~\ref{tab:regretbounds}). Without an explicit calculation of the stochastic constant, we show the interest of considering an average case through plotting the average constant $\bar \gamma_t$. We observe that although in theory, the bound of ONS appears better than the bound of OGD ($\mathcal O(\log(n)/\gamma)$ v/s $\mathcal O(\sqrt{n})$), when $\gamma_t$ goes to $0$, the bound of ONS is not $\mathcal O (\log(n))$, but $\mathcal O (\log(n)/\min_t \gamma_t )$. A similar finding in logistic regression has been made rigorous  by Hazan et al. \cite{hazan2014logistic} with the help of lower bounds matching $\mathcal O (\log(n)/\min_t \gamma_t )$. In practical applications, it is essential to consider more detailed analyses that remain robust in scenarios where $\min_t \gamma_t$ goes to $0$, which is what we propose with SurvONS and the stochastic approach.

\section{Simulation experiments}
\label{sec:sim}
In this section we present simulation results of our method. We considered a number of individuals $N = 10\,000$ and a number of iterations $n = 1\,000$. Then we sample a multivariate random normal of dimension $(N,d-1)$ with $d=4$ and mean vector and covariance matrix: \[ \eta: =\begin{pmatrix}  0 \\0\\0 \end{pmatrix}, \qquad \qquad \Sigma := \begin{pmatrix} 1 & 0 & 0 \\ 0 & 1 & 0\\ 0 & 0& 1 \end{pmatrix}\,. \] We add an intersect column that transforms the matrix into one of dimension $(N,d)$. This matrix corresponds to the covariates information $\{x_i\}_{i=1}^N$, which does not depend on time. The real parameter $\theta^*$ is set randomly following a $\mathcal N(0,I_d)$ distribution. We sample the arrival times $\tau_i$ as a uniform between $0$ and $n$ and we simulate $T_i$ and $C_i$  following an exponential distribution of rate $\exp(\theta^{*T} x_i)$, 
\[T_i \sim \tau_i + \exp(\exp(\theta^{*T} x_i)), \qquad C_i \sim \tau_i + \exp(\exp(\theta^{*T}x_i)). \] 
For more details on the common use of exponential distributions in survival analysis we refer to Selvin \cite{selvin2008survival}. We repeat this procedure $100$ times, and the results are the average curves over the $100$ data simulations. Additionally, we consider two exponential grids for the aggregation methods of size $K = 10$. First, we test a random grid, consider the SurvONS predictions $\theta_t$ for $t=1,\ldots,n$ , and define:
\[ G := \max_{t =1,\ldots,n}  ||\nabla \ell_t(\theta_t) || \, . \]
This process is repeated multiple times to ensure the stability of $G$ estimation.
Second, we choose $10$ equidistant points and then we generate the grids by considering the exponential of each point, such that:
\[ \Gamma_1 := (1/\sqrt{n},\ldots,1/4GD )\, , \qquad\Gamma_2 := (1/GD, \ldots ,10/GD )\, , \]
where $D$ is adjusted a posteriori such as $D := 1.1 ||\theta^*||$. Throughout this section we compare the results of the two choices of grid $\Gamma_1$ and $\Gamma_2$.


\begin{figure}[t]
\centering
    \includegraphics[scale=.25]{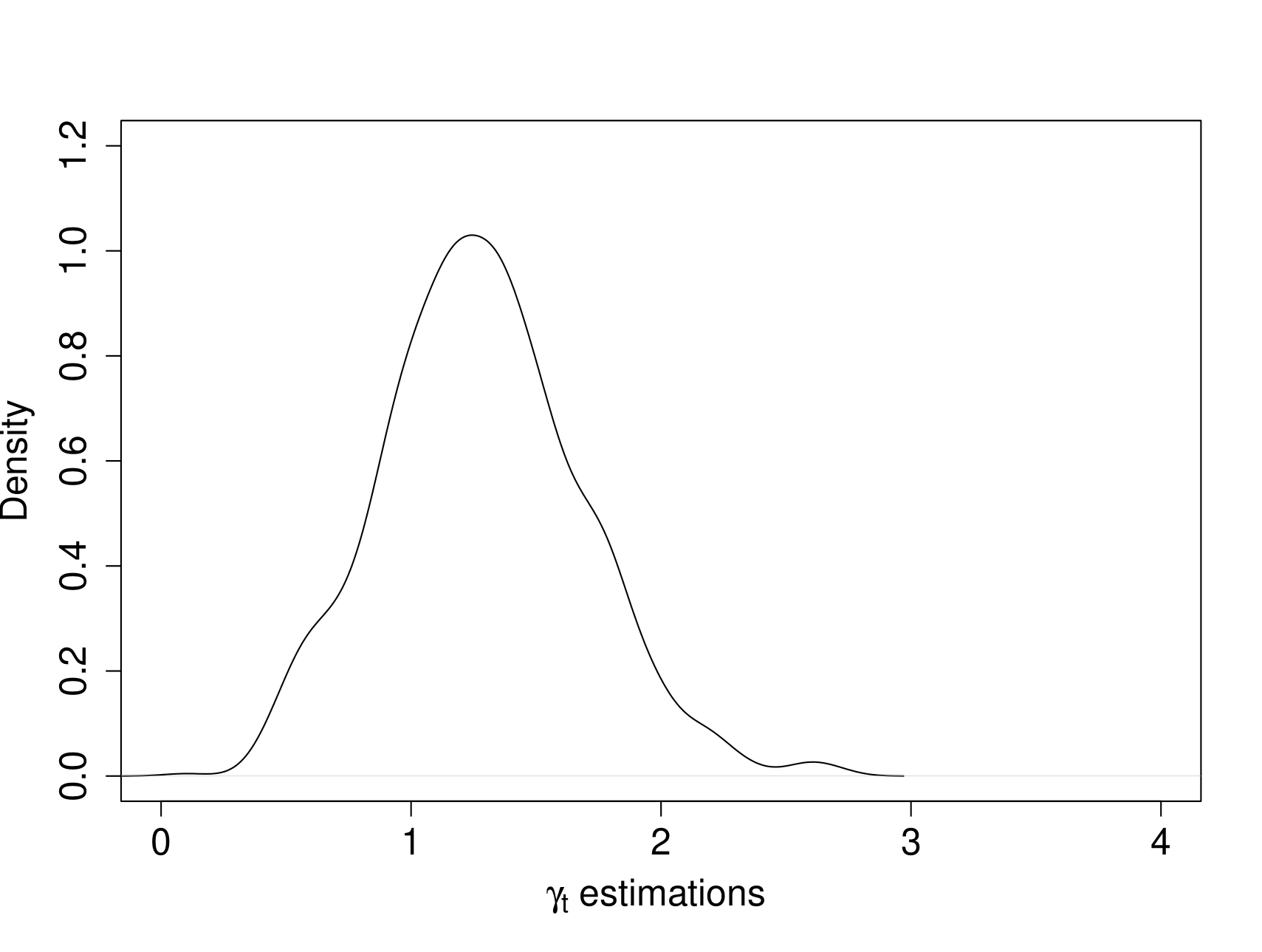} 
    \qquad
   \includegraphics[scale=.25]{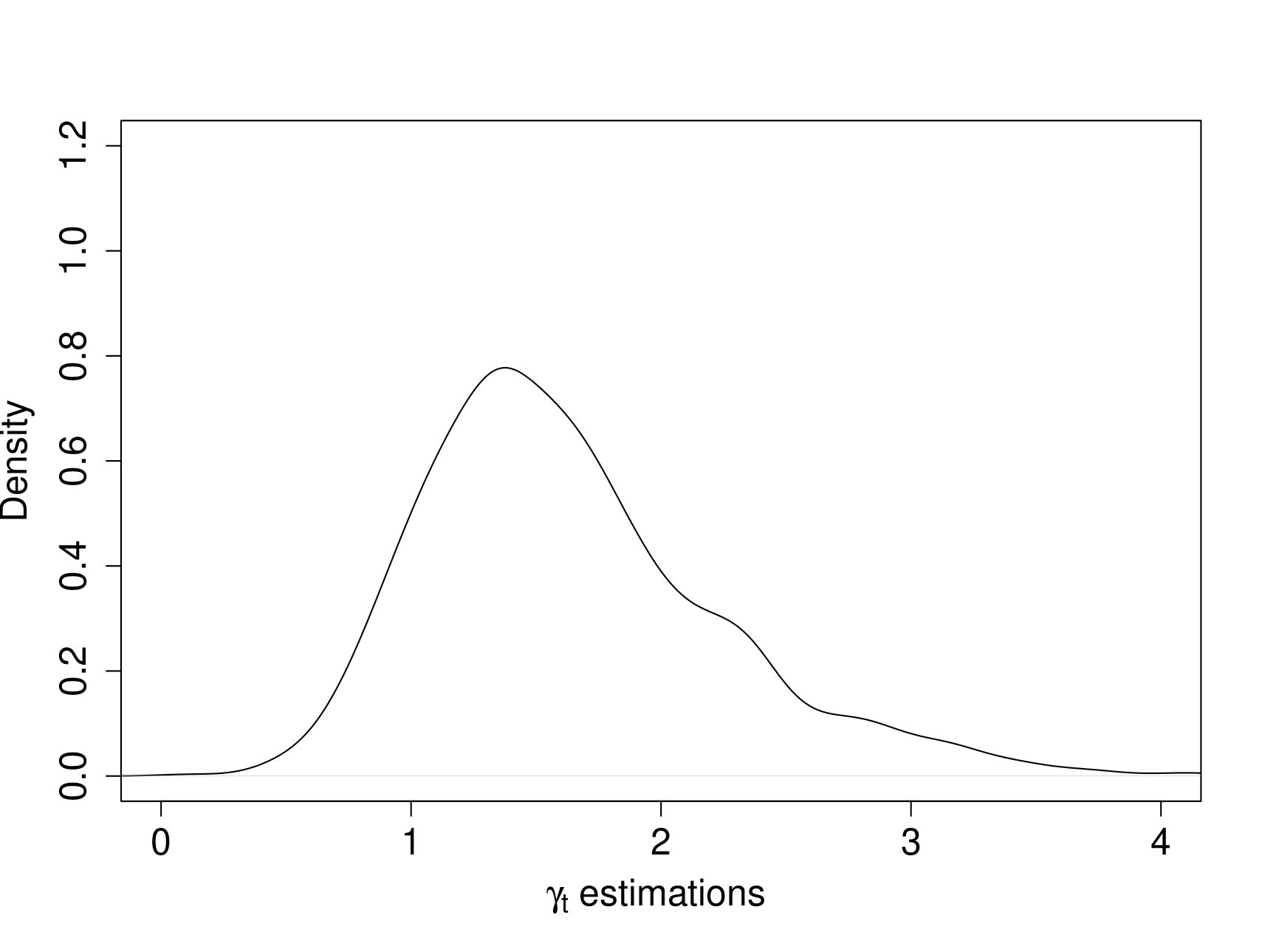}
    \caption[.]{
    Density of $\protect\gamma_t$ estimation obtained by Algorithm~\ref{alg:BOAm}, with $\Gamma_1$ [left] and $\Gamma_2$ [right] }
    \label{fig:gamma}
\end{figure}

We observe in Figure~\ref{fig:gamma}
the distribution of the average $\gamma_t$ estimations that we obtained from SurvONS. The average for $\Gamma_1$ is $1.24$ and $1.64$ for $\Gamma_2$. The similarity between both estimations elucidates the proximity of the graphs in Figure~\ref{fig:theoreticbound}, which is unsurprising given that the directional derivative constant is inherent to the loss function and does not depend on the algorithm or the selected grids.

We compare SurvONS, described in Algorithm~\ref{alg:BOAm}, with the BOA-ONS proposed by Wintenberger \cite{wintenberger2017optimal}. Additionally, we fit several ONS and OGD with constant learning rate equal to each $\gamma$ in the grid, and then we select the one that performs better to include in the comparison. Remark that this procedure overestimates the performances of ONS and OGD. We show the average cumulative difference between the negative log-likelihood of the estimations and the real parameters in Figure~\ref{fig:likelihood}.


\begin{figure}[t]
\centering
\includegraphics[scale=.25]{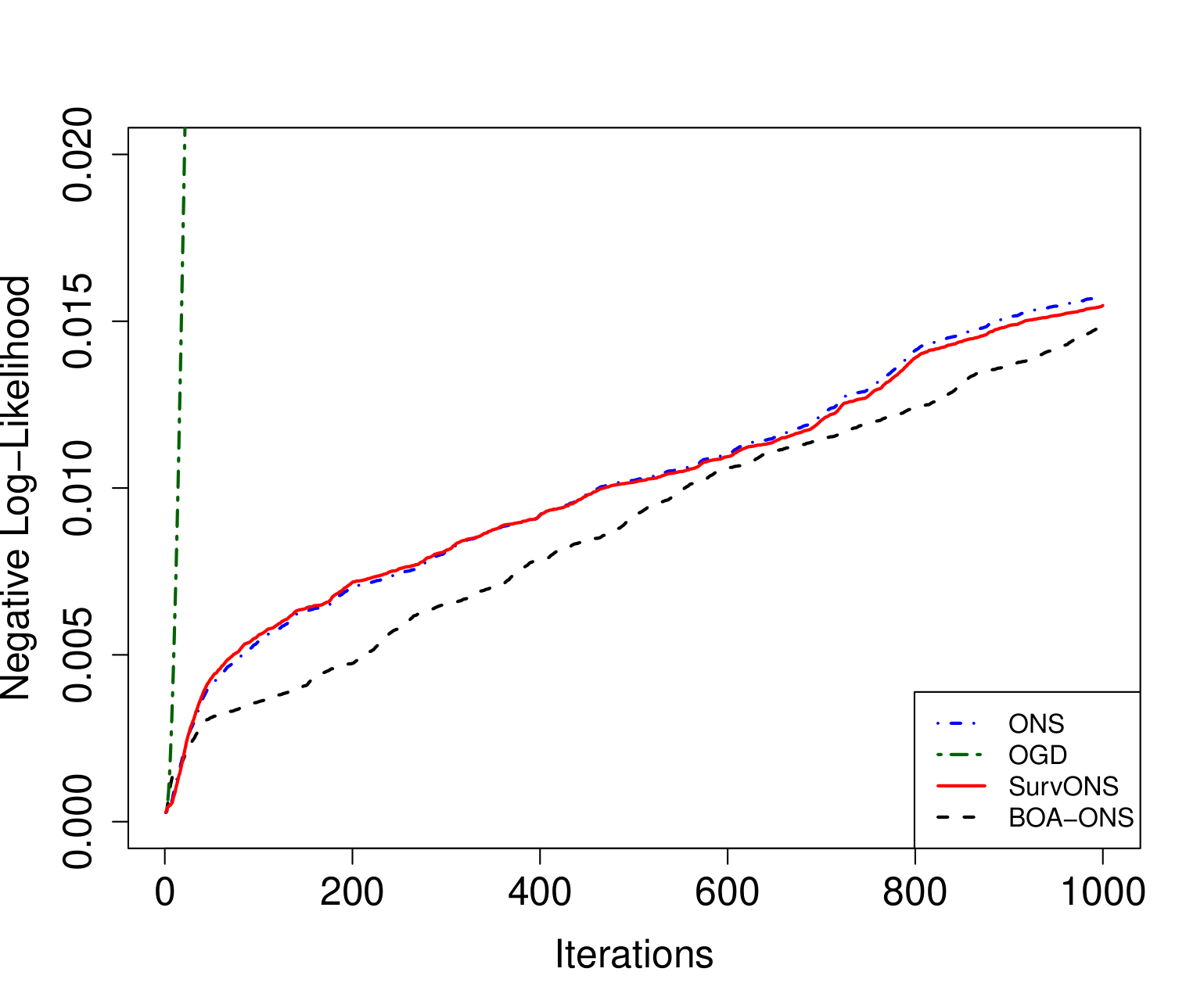} 
    \qquad
\includegraphics[scale=.25]{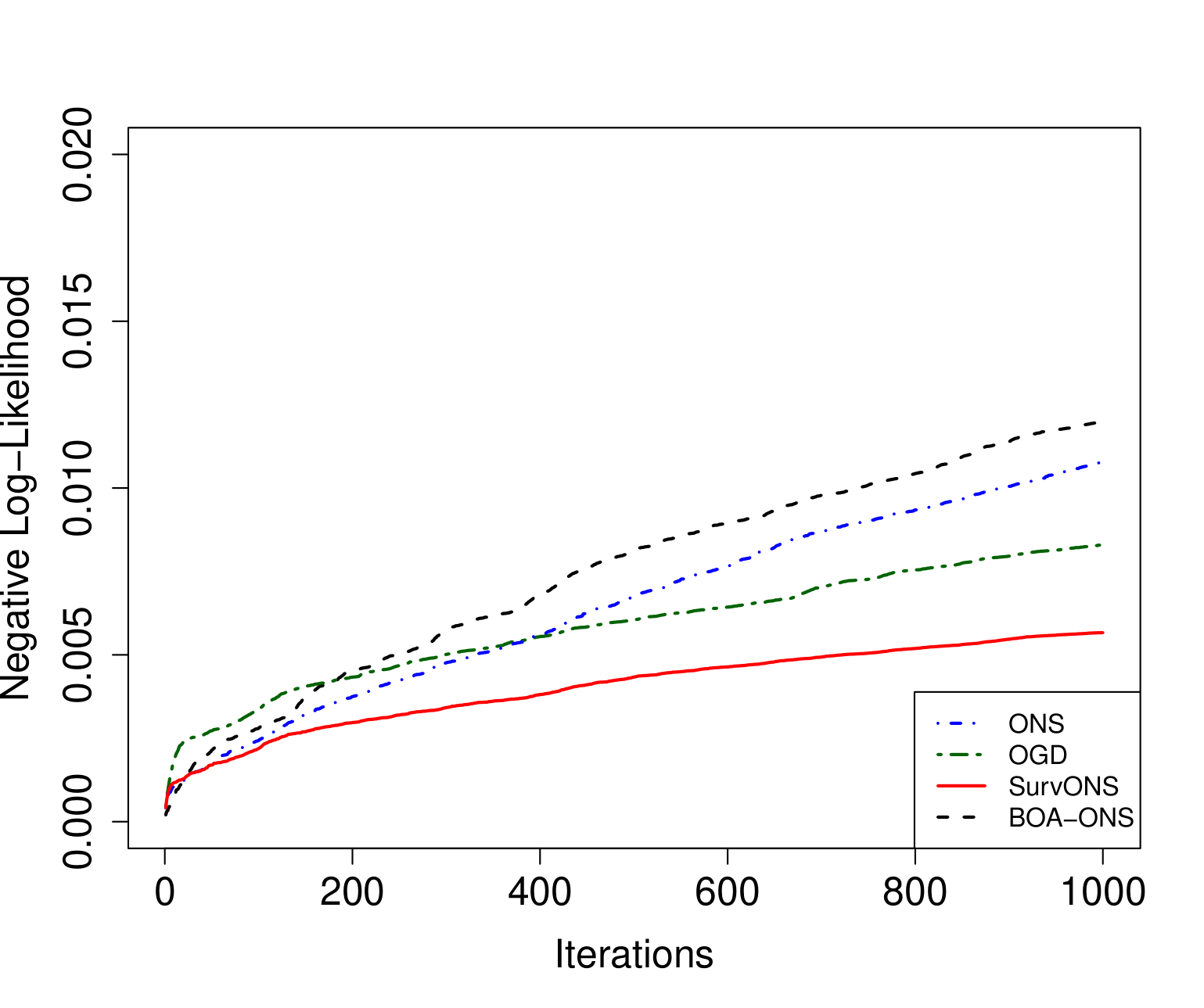}
    \caption[.]{Cumulative negative log-likelihood with hyperparameters in grid $\protect\Gamma_1$ [left] and $\Gamma_2$ [right] }
    \label{fig:likelihood}
\end{figure}

In Figure~\ref{fig:likelihood}, we observe that SurvONS (in purple) does not outperform BOA-ONS (in black) with the $\Gamma_1$ grid. However, the scenario changes with the second grid, $\Gamma_2$, where SurvONS proves to be more effective than the other methods. This unexpected result arises from the fact that the $\Gamma_1$ grid falls within the theoretical limits. Nevertheless, we observe a consistent improvement in performance for all algorithms when considering a larger grid. This discrepancy could arise from either an overestimation of the constant $G$ or the presence of outlier points exhibiting extremely large gradients. Nonetheless, given the similarity in the constant $\gamma_t$ estimation across the two grids, shown in Figure~\ref{fig:gamma}, we recommend opting for larger grids, ranging from 4 to 40 times the theoretical bound of $1/4GD$.

In addition, Figure~\ref{fig:error} presents the quadratic error, where we consider the cumulative average of the estimations. Specifically, given a sequence of algorithm predictions $(\theta_s)_{s=1}^t$, the cumulative average is defined as $\bar\theta_t := t^{-1}\sum_{s=1}^t \theta_s$. Let us remind that the curves depicted represent the average of $100$ instances obtained from simulating $100$ datasets. The figure is in log-log scale.


\begin{figure}[t]
\centering
\includegraphics[scale=.25]{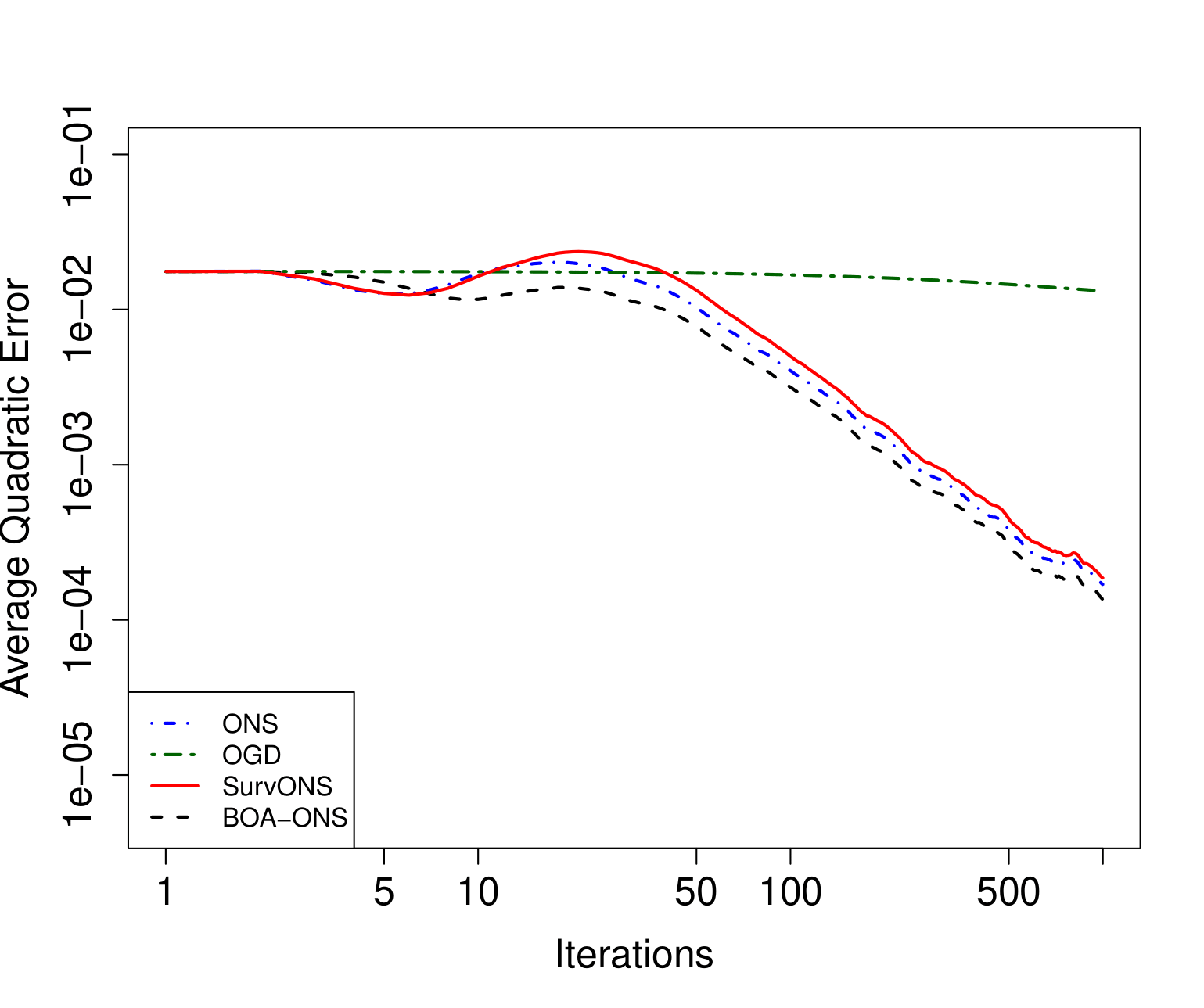} 
    \qquad
\includegraphics[scale=.25]{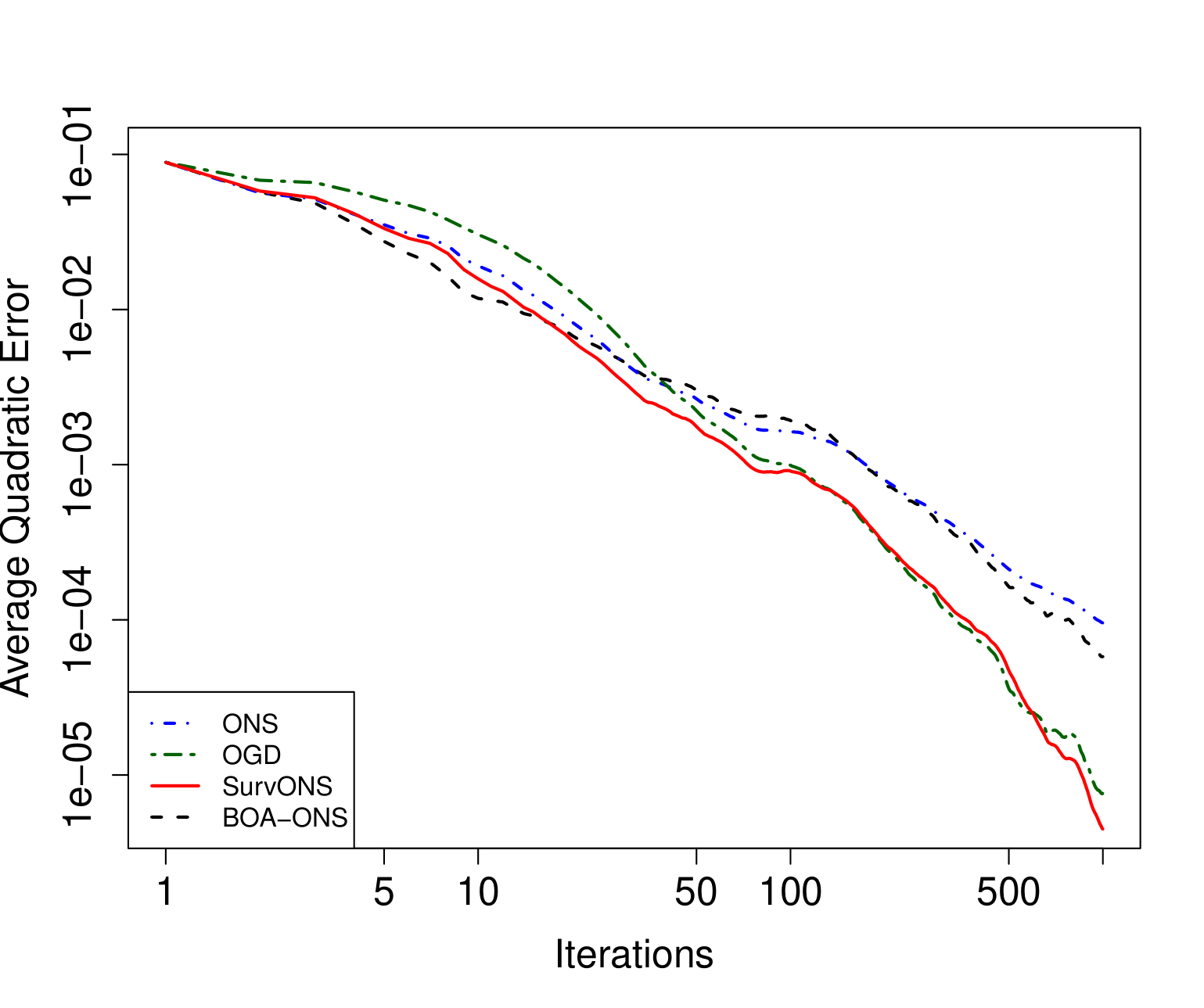}
    \caption[.]{Average quadratic error with hyperparameters in $\protect\Gamma_1$ [left] and $\Gamma_2$ [right] }
    \label{fig:error}
\end{figure}

Figure~\ref{fig:error} corroborates the result of Corollary~\ref{cor:stc}, which establishes the convergence of the estimations $\theta_t$ to the real parameter $\theta^*$ when using the ONS algorithm. The findings of Wintenberger in \cite{wintenberger2021stochastic}, which demonstrate the $\mathcal O (\log(n))$ stochastic regret of BOA-ONS, together with the insights from Figure~\ref{fig:error}, suggest the potential to extend a similar corollary to both BOA-ONS and SurvONS. Furthermore, Corollary~\ref{cor:stc}
can be easily extended to BOA-ONS by replacing the application of Theorem~\ref{thr:stcregret} with Theorem~\ref{thr:olivier} from \cite{wintenberger2021stochastic}.

\section{Conclusions}
In this paper, we presented a detailed mathematical framework for online survival data, analyzing the regret of Online Newton Step and its sensitivity to the learning rate. Notably, we found that tuning this parameter is challenging, and the regret bound is highly sensitive to its adjustment. Our first contributions is introducing a stochastic setting to ensure that ONS achieves logarithmic stochastic regret in the survival context. Additionally, we proposed an adaptive method, SurvONS, which aggregates ONS with different learning rates. Adaptive methods, commonly used in first-order algorithms like AdaGrad~\cite{duchi2011adaptive} or Adam~\cite{adam2014method}, offer a promising avenue for enhancing second-order algorithms. Our approach leverages adaptive strategies to improve efficiency and convergence, extending its applicability beyond the online survival domain. The regret analysis of SurvONS strategically selects larger learning rates to address sub-optimal parameters. In conclusion, aggregation methods enhance robustness in selecting algorithm hyperparameters; however, achieving and maintaining fast rates remains a non-trivial task.

Finally, in the simulation experiments, we compared two grid choices. Figure~\ref{fig:gamma} shows that $\gamma_t$ estimations closely align within the grids, and the second grid produces values that do not approach zero to the same extent as the first grid. Additionally, Figure~\ref{fig:likelihood} indicates that choosing larger values for the learning rate grid accelerates convergence, suggesting the preference for larger grids.

\clearpage
\bibliography{ref}


\appendix

\section{Background on parametric inference} 
\label{AppA}
\subsection{Proof of Proposition~\ref{likelihood}}

\begin{proof}
We define the equivalent of the survival probability for the censored distribution \[ G(t|x_i,\tau_i) = \mathbb P (c_i \geq t|x_i,\tau_i)\,.\]
Given $\theta \in \Theta$, we write the density of $u_i$ distinguishing two cases: $$ \mathbb{P}( u_i \in [t, t+h) ,\delta_i = 1 | \theta, x_i, \tau_i) = \mathbb{P}(t_i \in [t, t+h), c_i \geq t |x_i, \tau_i,\theta),$$ and $$ \mathbb{P}(u_i \in [t, t+h), \delta_i = 0 | \theta, x_i, \tau_i) = \mathbb{P}(c_i \in [t,t+h), t_i \geq t | \theta, x_i, \tau_i).$$ By conditional independence we obtain
\begin{align*}
\mathbb{P}( u_i \in [t, t+h) ,\delta_i = 1 | \theta, x_i, \tau_i) & = \mathbb{P}(t_i \in [t, t+h)|\theta, x_i, \tau_i)\mathbb{P}(c_i \geq t+h|\theta, x_i, \tau_i) ,\\ 
\mathbb{P}(u_i \in [t, t+h), \delta_i = 0 | \theta, x_i, \tau_i) & = \mathbb{P}(c_i \in [t,t+h)| \theta, x_i, \tau_i)S( t+h | \theta, x_i, \tau_i).
\end{align*}
When $h$ goes to zero, it tends respectively to $$G(t|\theta, x_i, \tau_i) f(t|\theta, x_i, \tau_i), $$ and $$g(t|\theta, x_i, \tau_i)\mathbb{P}(t_i \geq t| \theta, x_i, \tau_i).$$ Therefore, by the independence of the random variables $(t_i,c_i)$ among the events $i\in\{1,\ldots,N\}$ we obtain the density 
\begin{multline*}
f_{(u_i,\delta_i)_{1\le i\le N}}((u_i,\delta_i)_{1\le i\le N}\mid \theta, x_i, \tau_i) =\\
\displaystyle  \prod_{i=1}^N g(u_i|\theta, x_i, \tau_i)^{1-\delta_i} G( u_i |\theta, x_i, \tau_i)^{\delta_i}\prod_{i=1}^N f(u_i|\theta, x_i, \tau_i)^{\delta_i} S( u_i |\theta, x_i, \tau_i)^{1-\delta_i}\,.\end{multline*} Here we use the assumption of non-informative censoring (see Kalbfleisch et al.~\cite{kalbfleisch2011statistical}), which means that the censored distribution does not involve the parameter $\theta$. Then we obtain a simplified version of the likelihood, up to a multiplicative constant
\[
\ell(\theta)\propto \prod_{i=1}^N f(u_i|\theta, x_i, \tau_i)^{\delta_i} S(u_i |\theta, x_i, \tau_i)^{1-\delta_i}\,.
\]
Omitting an additional constant, we can equivalently write the log-likelihood to be \begin{equation*} \log(\ell(\theta)) = \displaystyle \sum_{i=1}^N \delta_i \log(f(u_i|\theta, x_i, \tau_i)) + (1- \delta_i)  \log(S( u_i|\theta, x_i, \tau_i)). \end{equation*}

Let us remark that $f(t|\theta, x_i, \tau_i) = H(t|x_i,\tau_i) S( t|x_i, \tau_i)$ and from the definition of $H(t|x_i,\tau_i)$ we can write the log-likelihood as 
\[ \log(\ell(\theta)) = \displaystyle \sum_{i=1}^N \delta_i \log(H(u_i|x_i,\tau_i)) - \int_{\tau_i}^{u_i} H(s|x_i,\tau_i)ds. \]

Following the exponential model of Definition~\ref{ass:exp} we replace $H(t|x_i,\tau_i)$ in the previous equation to get  $$ \log(\ell(\theta)) = \displaystyle \sum_{i=1}^N \delta_i \theta^T x_i(u_i) - \int_{\tau_i}^{u_i}\exp(\theta^T x_i(s))ds\, , $$ 
We write the negative log-likelihood:
\begin{equation*}\ell(\theta) = - \log(\ell(\theta)) = \displaystyle \sum_{i=1}^N -\delta_i \theta^T x_i(u_i) + \int_{\tau_i}^{u_i}\exp(\theta^T x_i(s))ds.\end{equation*} 
\end{proof}

\section{Online Convex Optimization} 
\label{AppB}
\subsection{Proof of Lemma~\ref{lem:1}}

\begin{proof}
Only the second assertion needs to be proven, the first one being Lemma 4.2.1 from Hazan \cite{hazan2022introduction} is already showed.  
To prove the second assertion we first see that Equation \eqref{eq:lem1} means that for all $\theta \in \Theta$ \[ \nabla ^2 \ell(\theta) \succcurlyeq \mu \nabla \ell(\theta) \nabla \ell(\theta)^{\top}, \] which implies that for all vector $\nu \in \mathbb{R}^d$ \[
\nu^{\top} \nabla^2\ell(\theta)\nu \geq \mu~\nu^{\top} \nabla \ell(\theta) \nabla \ell(\theta)^{\top}\nu.
\] Since $\nabla \ell(\theta) \nabla \ell(\theta)^{\top}$ is a rank one matrix and $\nu = \nabla \ell(\theta)$ is an eigenvector associated to the unique non-null eigenvalue, we can replace $\nu$ in the previous equation to get \[
\nabla \ell(\theta)^{\top} \nabla^2\ell(\theta)\nabla \ell(\theta) \geq \mu \nabla \ell(\theta)^{\top} \nabla \ell(\theta) \nabla \ell(\theta)^{\top}\nabla \ell(\theta).
\] When $\nabla \ell(\theta) \neq 0$ we can write  \[ \mu \leq \frac{\nabla \ell(\theta)^{\top} \nabla^2\ell(\theta)\nabla \ell(\theta)}{ || \nabla \ell(\theta)||^4},
\] and as this is true for every $\theta \in \Theta$ we have \[ \mu \leq \min_{\theta \in \Theta} \frac{\nabla \ell(\theta)^{\top} \nabla^2\ell(\theta)\nabla \ell(\theta)}{ || \nabla \ell(\theta)||^4}.
\]
\end{proof}

\subsection{Proof of Lemma~\ref{lem:2}}


\begin{proof}
The proof starts similarly than the one of Lemma 4.2.2 of Hazan \cite{hazan2022introduction}. We consider the concave function $p(\theta)=\exp(-\mu \ell(\theta))$. We derive that for $\theta_1,\theta_2 \in \Theta$:
\begin{align*}
 \ell(\theta_2)&\ge \ell(\theta_1)-\dfrac1{\mu}\log(1-\mu (\nabla \ell(\theta_1)^T(\theta_2-\theta_1)))\\   
 &\ge \ell(\theta_1) + \nabla \ell(\theta_1)^T (\theta_2-\theta_1)\\
 &\hspace{2cm}-\Big(\dfrac1{\mu}\log(1-\mu (\nabla \ell(\theta_1)^T(\theta_2-\theta_1)))+ \nabla \ell(\theta_1)^T (\theta_2-\theta_1) \Big)\,.  
\end{align*}
Using the Cauchy-Schwarz inequality we upper bound $|\nabla \ell(\theta_1)^T (\theta_2-\theta_1) | \le \|\nabla \ell(\theta_1)\| D$ for any $\theta_2\in \Theta$. Combined with the monotonicity of the function $\mu^{-1}\log(1-\mu z)+z$ which is decreasing for any $-\|\nabla \ell(\theta_1)\| D\le  z\le \|\nabla \ell(\theta_1)\| D$ we obtain:
\[
 \ell(\theta_2)\ge \ell(\theta_1) + \nabla \ell(\theta_1)^T (\theta_2-\theta_1)-\dfrac1{\mu}\log(1+\mu \|\nabla \ell(\theta_1)\| D)+\|\nabla \ell(\theta_1)\| D\,.
\]
By definition of the directional derivative constant, we thus can estimate:
\begin{align*}
\gamma&\le \min_{\theta_1,\theta_2\in\Theta} 2\frac{-\dfrac1{\mu}\log(1+\mu \|\nabla \ell(\theta_1)\| D)+\|\nabla \ell(\theta_1)\| D}{(\nabla \ell(\theta_1) (\theta_2-\theta_1))^2} \, ,\\
&\le \min_{\theta_1\in \Theta } 2\frac{-\dfrac1{\mu}\log(1+\mu \|\nabla \ell(\theta_1)\| D)+\|\nabla \ell(\theta_1)\| D}{(\|\nabla \ell(\theta_1)\|D)^2}  
\end{align*}
by another application of the Cauchy-Schwarz inequality.
\end{proof}

\subsection{The Online Newton Step  algorithm}

\begin{algorithm}[H]
\caption{ Online Newton Step \cite{hazan2007logarithmic}}
    \begin{algorithmic}
    \STATE {\bfseries Input: }$(\ell_t)_{t=1,2,\ldots}, \gamma>0, n\geq 1, \epsilon>0$\\
    \STATE {\bfseries Initialization:}  $ \theta_0 \in \Theta, ~ A^{-1}_0 = (1/\epsilon) \mathds{1}_d$\\
    \FOR{iteration $t=1,\ldots,n$}

        \STATE \textbf{Observe:} $\nabla\ell_{t}(\theta_t)$
        \STATE \textbf{Recursion:} $$A^{-1}_t = A^{-1}_{t-1} - \frac{A_{t-1}^{-1} \nabla \ell_t(\theta_t) \nabla \ell_t(\theta_t)^T A_{t-1}^{-1}}{1+ \nabla \ell_t(\theta_t) A_{t-1}^{-1} \nabla \ell_t(\theta_t)^T}$$\\ $$\theta_{t+1} = {\rm Proj}_t\Big(\theta_t - \frac{1}{\gamma} A_t^{-1} \nabla \ell_t(\theta_t)\Big) $$
        \STATE where ${\rm Proj}_t(\theta^*) \in \displaystyle \arg\min_{\theta\in \Theta}(\theta - \theta^*)^T A_t (\theta -\theta^*)$.
    \ENDFOR
    \RETURN $\theta_n$
    \end{algorithmic}
    \label{alg:ONS}
\end{algorithm}

\section{Stochastic Setting}
\label{AppC}
In this section we prove Theorem~\ref{thr:stcregret}, and for this we need to recall the hypothesis of Theorem 7 from \cite{wintenberger2021stochastic}.

\textbf{(H1)} The diameter of $\Theta$ is $D$ and the loss functions $\ell_t$ are continuously differentiable over $\Theta$ a.s. with integrable gradients.
 
\textbf{(H2)} The random loss functions $(\ell_t)_{t=1,2,\ldots}$ are stochastically exp-concave~\ref{con:1} for some $\gamma \geq 0$.
 
\textbf{(H3)} The gradients $(\nabla \ell_t(\theta_t))_{t=1,2,\ldots}$, satisfy for $G_1,G_2>0$ and all $ k \geq 1, t=1,2,\ldots$, and $ \theta \in \Theta$: 
    \begin{align*}
    \mathbb E_{t-1}[(\nabla \ell_t(\theta_t)^{\top} (\theta_t - \theta))^{2k}]& \le k! (G_1 D)^{2(k-1)} \mathbb E_{t-1}[ (\nabla \ell_t(\theta_t)^{\top}(\theta_t - \theta))^2] \qquad a.s., \\
     \mathbb E_{t-1} [ || \nabla \ell_t (\theta) ||^{2k} ]& \le k! G_1^{2(k-1)}\mathbb E_{t-1} [ || \ell_t (\theta_t) ||^2]  \qquad a.s.,\\
     \mathbb E_{t-1} [ || \nabla \ell_t (\theta) ||^2 ]&  \le G_2^2 \qquad \qquad a.s.   
\end{align*}

Let us notice that condition \textbf{(H3)} 
is satisfied in the bounded cases\\ $||\nabla \ell_t(\theta_t)||^2 \leq G^2, ~ t=1,2,\ldots$ with $G_1:=G_2:=G$. Condition \textbf{(H3)} is independent on the risk $L_t(\theta_t) = \mathbb E_{t-1} [\ell_t(\theta_t)],~t=1,2,\ldots$, and thus, it does not interfere with condition \textbf{(H2)}. Additionally, we notice that in our setting where we consider the stochastic losses $\ell_t$ defined in \eqref{stcloss}, the hypothesis \textbf{(H1)} is already satisfied. Now, we recall the stochastic regret bound theorem.

\begin{Theorem}[Wintenberger \cite{wintenberger2021stochastic}]\label{thr:olivier}
    Under \textbf{(H1)}, \textbf{(H2)} with constant $\gamma$ and \textbf{(H3)}, for $\varrho >0$ and $n \geq 1$ the ONS algorithm~\ref{alg:ONS} with learning rate $\gamma/3$ satisfies with probability $1-3 \varrho$ the stochastic regret bound:
    \begin{align*}
    Risk_n &\leq \frac{3}{2\gamma} \left(1+d\log\left(1+\frac{2(\gamma D)^2 \left(nG_2^2 + G_1^2 \log(\varrho^{-1}) \right) }{9} \right) \right) \\
    & \hspace{6cm} + \left( \frac{4 \gamma (G_1 D)^2}{9} + \frac{18}{\gamma} \right) \log(\varrho^{-1})\,.
    \end{align*}
\end{Theorem}

In order to simplify notation we refer to the right-hand-side bound as $\mathcal B (n)$. In addition we need a proposition presented in \cite{wintenberger2021stochastic} that gives us a constant $\gamma$ that assures the stochastic exp-concavity of the losses.

\begin{proposition}[Wintenberger \cite{wintenberger2021stochastic}]
\label{prop:olivier}
 If  $L_t$ is $\mu$-strongly convex and there exists $G>0$ such that \begin{equation*} G^2 I_d \succcurlyeq \mathbb{E}_{t-1}[\nabla \ell_t(\theta) \nabla \ell_t(\theta)^T], \qquad \qquad \forall \theta \in \Theta, a.s., t=1,2,\ldots\, ,
\end{equation*} then Definition~\ref{con:1} holds with $\gamma := \mu/G^2$.
\end{proposition}

In the ideal case, we would like to prove that $\ell_t$ satisfies the conditions \textbf{(H2)} and \textbf{(H3)}. To prove \textbf{(H2)} we can use Proposition~\ref{prop:olivier} if we find a constant such that the loss is strongly convex and a constant that bounds the expectation of the gradients. Unfortunately, we are not able to find this last constant a.s. but, proving a weaker version of \textbf{(H3)} we can define an auxiliary loss function that satisfies all the hypothesis and allows us to prove Theorem~\ref{thr:stcregret}.

First, we prove that with high probability there is a constant $G$ that upper bounds the norm of the gradients of $(\ell_t)_{t=1,2,\ldots}$, this corresponds to the weaker \textbf{(H3)}. Secondly, we prove that the conditional risks $(L_t)_{t=1,2,\ldots}$ are strongly convex for some constant $\mu$, which consists of finding a lower bound of $\nabla^2 L_t(\theta)$ that does not depend on $\theta$ and $t$. This corresponds to only one of the conditions of Proposition~\ref{prop:olivier}, necessary to prove \textbf{(H2)}. Finally, we show how to use weak \textbf{(H3)} and half of \textbf{(H2)} to prove Theorem~\ref{thr:stcregret}.

\subsection{Upper bound (H3)}
We want to find an upper bound for $||\nabla \ell_t(\theta)||^2$ and for this we first define for all $t=1,2,\ldots$
\[R_t = \displaystyle \sum_{i=1}^{N_t} r_{it} \qquad \text{where}\qquad r_{it} = \mathds{1} \{\tau_i \leq t, u_i > t-1 \}\, , \] 
and where $N_t$ is the count function of the Poisson process defined in Section~\ref{sthsection}. Following, we prove that for all $t=1,2,\ldots$, $R_t$ is upper bounded with high probability.

\begin{Lemma} \label{lem:prop4} 
 Let $\varrho >0$. Then, with probability at least $1-\varrho$, for all $t=1,2,\ldots$, we have
 \[R_t \leq 32e^{Dx_{\infty}}(4\lambda+1+\log(2/\varrho)) \,.\]    
\end{Lemma}
\begin{proof}
   Since $u_i = \min\{c_i,t_i\} \leq t_i$ we can upper bound $R_t \leq \displaystyle \sum_{i=1}^{\infty} \mathds{1}\{t_i \geq t-1\}\mathds{1}\{\tau_i \leq t\}$. Then, we define $A_t = \{ i: \tau_i \leq t\}$ and $Z_t = \displaystyle \sum_{i \in A_t} \mathds{1}\{t_i \geq t-1\} $ and therefore, it will be enough to find a bound to $Z_t$ to conclude.
   Given a constant $z >0$ and $m \geq 1$, we fix $N_t \leq m$ and we first upper bound the conditional probability
   \begin{align*}
 \mathbb P(Z_t \geq z \vert N_t = m) & = \mathbb P \left(\sum_{i \in A_t} \mathds{1}\{t_i \geq t-1\}\geq z \Big \vert N_t = m \right)\,.\end{align*}
 Let us notice that $N_t = |A_t|$. We would like to apply the concentration inequality of Chernoff-Hoeffding to the sum of Bernouilli random variables $\mathds{1}\{t_i \geq t-1\}$ (see Hoeffding~\cite{hoeffding1994probability}), and for this we need to upper bound 
   \begin{align*}
       &\mathbb P(t_i \geq t-1 \vert i \in A_t, x_i )\\
       & \hspace{2cm} = \sum_{s=1}^t \mathbb P(t_i \geq t-1 \vert s-1 \leq \tau_i \leq s,x_i ) \mathbb P(s-1 \leq \tau_i \leq s \vert 0 \leq \tau_i \leq t )\\
       & \hspace{2cm} = \frac{1}{t} \sum_{s=1}^t \mathbb P(t_i \geq t-1 \vert s-1 \leq \tau_i \leq s,x_i )\, ,
    \end{align*}
where we use the uniform distribution of the Poisson process points given an interval (for more details on Poisson processes, see Daley and Vere-Jones~\cite{daley2003introduction}). Then, by the definition of the survival function (see Section~\ref{paramsection}) we get 

 \begin{align*}
    \mathbb P( t_i \geq t-1 \vert i \in A_t, x_i ) & \leq \frac{1}{t} \sum_{s=1}^t S(t-1 \vert s,x_i)\wedge 1\\
       & = \frac{1}{t} \sum_{s=1}^t \exp\left( -(t-1-s)e^{\theta^{*\top}x_i} \right)\wedge 1\\
       & \leq \frac{1}{t} \sum_{s=1}^t \exp\left( -(t-1-s)e^{-Dx_{\infty}} \right)\wedge 1\\
       & \leq \frac{1}{t} \sum_{s=-1}^{\infty} \exp\left( -se^{-Dx_{\infty}} \right)\wedge 1\\
       & =\frac{2-\exp(-e^{-Dx_{\infty}})}{t\left(1- \exp(-e^{-Dx_{\infty}})\right)}\\
       &  \leq \frac{4\,e^{Dx_{\infty}}}{t}\,.
   \end{align*}
In the last line we use that $1-\exp(-x)\ge x/2$ for $0\le x\le 1$.

 The Chernoff-Hoeffding's inequality gives us for any sequence $X_1,\ldots,X_m$ with $\mathbb E [X_i]\leq p$ and any $ \varepsilon >0$  \[ \mathbb P\left(\sum_{i=1}^m X_i \geq pm + \varepsilon \right) \leq \exp\left( -\frac{\varepsilon^2}{2mp(1-p)}\right) \leq \exp\left( -\frac{\varepsilon^2}{2mp}\right)\,.  \]
 
Applying this to the sum of the $\mathds{1}\{t_i \geq t-1\}$ with $\mathbb E [\mathds{1}\{ t_i \geq t-1\} \vert i \in A_t ] \leq  \frac{e^{2+Dx_{\infty}}}{t}$ given $|A_t| = m$ and using the conditional independence of the Poisson process points, we obtain \[ \mathbb P\left(Z_t \geq \frac{m}{t}4e^{Dx_{\infty}} + \varepsilon \Big \vert |A_t| = m\right) \leq \exp\left( - \frac{\varepsilon^2t}{8me^{Dx_{\infty}}}\right)\,.\]

Therefore, replacing $|A_t|$ by $N_t$
\[ \mathbb P \left(Z_t \geq \frac{m}{t}4e^{Dx_{\infty}} + \varepsilon \Big \vert N_t = m \right) \leq \exp\left( - \frac{\varepsilon^2t}{8me^{Dx_{\infty}}}\right)\,.\]

We set $z = \frac{m}{t}4e^{Dx_{\infty}} + \varepsilon$ with which we get $\varepsilon = z -\frac{m}{t}4e^{Dx_{\infty}}$ and \[\mathbb P\left(Z_t \geq z \Big \vert N_t = m\right) \leq \exp\left( - \frac{\left(z -\frac{m}{t}4e^{Dx_{\infty}}\right)^2t}{8me^{Dx_{\infty}}}\right)\,.\]
If we suppose $\frac{n}{t}4e^{Dx_{\infty}} \leq \frac{z}{2}$ we obtain \[ \mathbb P\left(Z_t \geq z \Big \vert N_t = m\right) \leq \exp\left( - \frac{z^2t}{32me^{Dx_{\infty}}}\right)\,.\]
With this we found a bound for the conditional probability of $Z_t$ being bigger than a certain constant. The next step is to bound $\mathbb P (Z_t \geq z)$, and for this we need to upper bound the probability of $N_t$ being large. Let $M>0$, since $N_t$ follows a Poisson distribution of intensity $\lambda t$, we can apply a Chernoff bound argument (more details in Mitzenmacher and Upfal~\cite{mitzenmacher2017probability}) obtaining \[\mathbb P(N_t > M) \leq \left(\frac{e\lambda t}{M}\right)^M e^{-\lambda t} \qquad \text{for} \qquad M > \lambda t\, ,\]
and \[\mathbb P(N_t > M) \leq e^{-M-\lambda t} \qquad \text{when} \qquad M > e^2\lambda t\, . \]
Now, we compute \begin{align*}
    \mathbb P (Z_t \geq z) &= \sum_{m=1}^M\mathbb P (Z_t \geq z\vert N_t =m)\mathbb P (N_t =m) + \mathbb P (Z_t \geq z\vert N_t > M)\mathbb P (N_t > M)\\
    &\le \sum_{m=1}^M\exp\left( -\frac{z^2t}{8me^{2+Dx_{\infty}}}\right)\mathbb P (N_t =m)\\
    & \hspace{5cm} + \mathbb P (Z_t \geq z\vert N_t > M)\mathbb P (N_t > M)\\
    & \leq \exp\left( -\frac{z^2t}{8Me^{2+Dx_{\infty}}}\right) + \exp(-M-\lambda t)\, ,
\end{align*}
where we use the bounds we previously found for $\mathbb P (Z_t \geq z\vert N_t = m)$ and $\mathbb P (N_t > M)$. Finally, we need to choose $z$ and $M$ such that $\mathbb P (Z_t \geq z) \leq \varrho/t^2$. Reminding the constrain $M > e^2\lambda t$, we want \[\exp(-M-\lambda t) \leq \frac{\varrho}{2t^2}\,,\]
which is true if \[M \geq \log\left(\frac{2t^2}{\varrho}\right) - \lambda t \,\]
and then we can choose $M = e^2\lambda t + \log(2t^2/\varrho)$ that satisfies both conditions. Similarly, we want \[ \exp\left( -\frac{z^2t}{32Me^{Dx_{\infty}}}\right) \leq \frac{\varrho}{2t^2}\,,\]
which is true if \[z \geq \sqrt{\frac{32Me^{Dx_{\infty}}}{t}\log(2t^2/\varrho)}\,,\]
and reminding the constrain $z \geq \frac{8M}{t}e^{Dx_{\infty}} $ we choose $z$ such that \[z \geq \frac{8M}{t}e^{Dx_{\infty}} + 2\sqrt{\frac{8Me^{Dx_{\infty}}}{t}\log(2t^2/\varrho)}\, . \]
Due to Young's inequality $a+2\sqrt{ab} \leq 2a + b$, we can also choose \[z \geq \frac{16M}{t}e^{Dx_{\infty}}+\log(2t^2/\varrho)\, ,  \]
which replacing $M$ becomes \begin{align*}
&\frac{16}{t}e^{Dx_{\infty}}(e^2\lambda t + \log(2t^2/\varrho))+\log(2t^2/\varrho)\\
&\hspace{4cm} = 16\lambda e^{Dx_{\infty}} +\left(1+\frac{16}{t}e^{Dx_{\infty}}\right) \log(2t^2/\varrho)\\
&\hspace{4cm} \leq  32e^{Dx_{\infty}}(4\lambda +1+\log(1/\varrho))\,.
\end{align*}
In conclusion, we choose $z =  32e^{Dx_{\infty}}(4\lambda+1+\log(1/\varrho))$ and we get \begin{align*}
 \mathbb P (Z_t \geq z) &= \mathbb P (Z_t \geq  32e^{Dx_{\infty}}(4\lambda+1+\log(1/\varrho)))\\
    & \leq \exp\left( -\frac{z^2t}{32Me^{Dx_{\infty}}}\right) + \exp(-M-\lambda t)\\
    & \leq \frac{\varrho}{2t^2} + \frac{\varrho}{2t^2}\\
    & = \frac{\varrho}{t^2}\,.
\end{align*}
Using an upper-bound over $t$ \[\mathbb P(\forall t=1,2,\ldots \qquad Z_t \geq  32e^{Dx_{\infty}}(4\lambda+1+\log(1/\varrho))) \leq \sum_{t=1,2,\ldots}\frac{\varrho}{t^2}=\varrho \frac{\pi^2}{6}\le 2\varrho \,, \]
which concludes the proof.
\end{proof}

Finally, we are now ready to give the desired upper bound for $||\nabla \ell_t(\theta)||^2$ in the following proposition

\begin{proposition}\label{upperbnd}Let $\varrho >0$. Then, with probability $1-\varrho$ we have
\[ || \nabla \ell_t (\theta)||^2 \leq G^2 \qquad \forall \theta \in \Theta , t=1,2,\ldots\, ,\] 
with $G := 32e^{Dx_{\infty}}(4\lambda+1+\log(2/\varrho))\big( 1 + e^{Dx_{\infty}} \big)  x_{\infty}$.
\end{proposition}

\begin{proof}
    Let us notice that $\nabla \ell_t (\theta) \in \mathbb R^d$ and recall
    \[ \nabla \ell_t(\theta) = \sum_{i=1}^{N_t}-y_{it}x_i + r_{it} \exp(\theta^{\top}x_i)x_i\left( u_i \wedge t - \tau_i \vee 0 \right)_+\,. \]
Then, we have \[ || \nabla \ell_t (\theta) || \leq \sum_{i=1}^{N_t} ||y_{it}x_i|| + ||r_{it} \exp(\theta^{\top}x_i)x_i \left( u_i \wedge t - \tau_i \vee (t-1) \right)_+||\, ,\]
noticing that $y_{it} \leq r_{it}$, $x_i \leq x_{\infty}$, $\exp(\theta^{\top}x_i) \leq \exp(Dx_{\infty})$\\ and $\left( u_i \wedge t - \tau_i \vee (t-1) \right)_+ \leq 1$,
\begin{align*}|| \nabla \ell_t (\theta) || & \leq \sum_{i=1}^{N_t} ||r_{it}x_{\infty}|| + ||r_{it} \exp(Dx_{\infty})x_{\infty} ||\\
& \leq \sum_{i=1}^{N_t} r_{it} \cdot \Big( 1 + \exp(Dx_{\infty}) \Big)  x_{\infty}\\
& \leq R_t \Big( 1 + \exp(Dx_{\infty}) \Big)  x_{\infty}\, , \end{align*}
by definition of $R_t=\sum_{i=1}^{N_t} r_{it}$. In consequence, \[ || \nabla \ell_t (\theta) ||^2 \leq \big( 32e^{Dx_{\infty}}(4\lambda+1+\log(2/\varrho)) \big( 1 + \exp(Dx_{\infty}) \big)  x_{\infty} \big)^2 \, , \] with probability $1 - \varrho$ and where the last inequality is due to Lemma~\ref{lem:prop4}. This conclude the proof.
\end{proof}

\subsection{Strong convexity \textbf{(H2)}}
Before showing the strong convexity let us remark that we can write $S(t|x_i,\tau_i) = \exp\left( - \displaystyle \int_{\tau_i}^t H(s|x_i,\tau_i)ds\right)$ and because $f(t|x_i,\tau_i) = H(t|x_i,\tau_i) S(t|x_i,\tau_i)$  the density of $t_i$ is given by
\[f(t|x_i,\tau_i) = H(t|x_i,\tau_i) \exp\left(-\int_{\tau_i}^t H(s|x_i,\tau_i)ds\right).\] 
Given $\theta^* \in \Theta$, the real parameter and replacing by our parametric model $h(t|x_i,\theta^*,\tau_i) = \exp(\theta^{*T} x_i)$ we have \begin{equation}
f(t|x_i,\tau_i) := \exp(\theta^{*T} x_i) \exp\left(- (t-\tau_i)\exp(\theta^{*T} x_i)\right ) \mathds{1}\{ t \geq \tau_i\} .  
\end{equation}
 We also denote 
\[
    \ell_t(\theta; s, c, x, \tau) := \Big( - \mathds{1}\{ t-1 < s \leq t \wedge c \}   \theta^T x  +   \exp(\theta^T x)\big((c \wedge s \wedge t) - (\tau \vee (t-1))\big)_+ \Big)
\]
and recalling that $y_{it} = \mathds{1}\{t-1 < t_i \leq t \wedge c_i\}$, $u_i = t_i \wedge c_i$ and $r_{it} = \mathds{1}\{\tau_i < t,  u_i \geq t-1\}$, we have
\begin{align*}
\ell_t(\theta) &= \sum_{i=1}^{N_t} -y_{it} \theta^T x_i  +  r_{it} \exp(\theta^T x_i)\big( (u_i \wedge t) - (\tau_i \vee (t-1))\big)\\
&= \sum_{i=1}^{N_t} \ell_t(\theta; t_i,c_i,x_i,\tau_i) \,.
\end{align*}
In addition, as $(t_i)_{i \geq 1}$ and $(c_i)_{i \geq 1}$ are i.i.d. we name $T$ and $C$ random variables that are distributed as $t_i$ and $c_i$, respectively. We first prove the following Lemma that gives us an explicit expression of the risk function $L_t(\theta):=\mathbb{E}_{t-1}[\ell_t(\theta)]$, $\theta\in\Theta$, $t=1,2,\ldots$.

\begin{Lemma}\label{risk} For every $t=1,2,\ldots$ and every $\theta\in \Theta$ the risk function is given by \begin{align*} &L_t(\theta)= \lambda \mathbb{E} \Big[(e^{(\theta-\theta^*)^T x}-\theta^Tx) \mathds{1}\{T\leq C\} (1-T)_+ \mid \tau = 0 \Big]\\ 
& +\sum_{ \substack{i: \{ u_i >t-1\}\\ i: \{ \tau_i \leq t-1\}  } }   \left(e^{(\theta-\theta^*)^{\top}x_i} - \theta^{\top}x_i \right) \mathbb{P}(\mathds{1}\{t-1+\tau_i < T \leq \tau_i + t \wedge C \}\vert x_i,\tau_i, \tau = 0 )\,.\end{align*} 
\end{Lemma}

\begin{proof}

The expected value is \[ \mathbb{E}_{t-1}[\ell_t(\theta)] = \mathbb{E}_{t-1} \left[ \sum_{i=1}^{N_t} \ell_t(\theta;t_i,c_i,x_i,\tau_i) \right],
\]
which we separate in two terms \begin{equation}\label{eq:sep} \mathbb{E}_{t-1}[\ell_t(\theta)] = \mathbb{E} \left[ \sum_{i=N_{t-1}}^{N_t} \ell_t(\theta;t_i,c_i,x_i,\tau_i) \right] + \sum_{i=1}^{N_{t-1}} \mathbb{E}_{t-1} \left[  \ell_t(\theta;t_i,c_i,x_i,\tau_i) \vert x_i,\tau_i \right]\,.  \end{equation}

Now, recalling that $g$ and $f$ respectively denote the conditional densities of $t_i$ and $c_i$ given $(\tau_i,x_i)$ and, because $c_i$ and $t_i$ are independent given $(\tau_i, x_i)$, we first calculate the first term
\begin{align*}
&\mathbb{E} \left[ \sum_{i=N_{t-1}}^{N_t} \ell_t(\theta;t_i,c_i,x_i,\tau_i) \right]\\
&\hspace{2cm} = \mathbb{E} \left[\sum_{i=N_{t-1}}^{N_t} \mathbb{E} \left[ \ell_t(\theta; t_i,c_i,x_i,\tau_i) | x_i,\tau_i\right]  \right]\\
& \hspace{2cm} = \mathbb{E} \left[\sum_{i=N_{t-1}}^{N_t} \int_{-\infty}^\infty \int_{-\infty}^\infty   \ell_t(\theta; s, c, x_i, \tau_i) g(c|x_i,\tau_i) f(s|x_i,\tau_i) ds dc \right]\,.
\end{align*} 

Now because $x_i$ are i.i.d. and independent from $\tau_i$ and $c_i$, denoting by $x$ a random variable with the same distribution we have
\begin{align*}
&\mathbb{E} \left[ \sum_{i=N_{t-1}}^{N_t} \ell_t(\theta;t_i,c_i,x_i,\tau_i) \right]\\
& \hspace{2cm} = \mathbb{E} \left[\sum_{i=N_{t-1}}^{N_t} \int_{-\infty}^\infty \int_{-\infty}^\infty   \ell_t(\theta;s,c, x, \tau_i) g(c|x,\tau_i) f(s|x,\tau_i) ds dc \right]
\end{align*}
which can be written as the stochastic integral 
\begin{align*}
&\mathbb{E} \left[ \sum_{i=N_{t-1}}^{N_t} \ell_t(\theta;t_i,c_i,x_i,\tau_i) \right]\\
    & \hspace{.3cm} = \mathbb{E} \left[\int_{t-1}^t \int_{-\infty}^\infty \int_{-\infty}^\infty   \ell_t(\theta; s,c,x, v) g(c|x,v) f(s|x,v) ds dc dN(v) \right] \\
    & \hspace{.3cm} = \lambda \mathbb{E} \left[\int_{t-1}^t \int_{-\infty}^\infty \int_{-\infty}^\infty   \ell_t(\theta;s, c, x, v) g(c|x,v) f(s|x,v) ds\, dc\, dv \right] \\
    & \hspace{.3cm} = \lambda \mathbb{E} \left[\int_{t-1}^t \int_{-\infty}^\infty \int_{-\infty}^\infty   \ell_t(\theta; s+v,c+v,x, v ) g(c+v |x,v) f(s+v|x,v) ds\, dc\, dv \right]\\
    & \hspace{.3cm} = \lambda \mathbb{E} \left[\int_{t-1}^t \int_{-\infty}^\infty \int_{-\infty}^\infty   \ell_t(\theta; s+v,c+v,x, v ) g(c |x,0) f(s|x,0) ds\, dc\, dv \right] \\ 
    & \hspace{.3cm} =  \lambda \mathbb{E} \left[ \int_{-\infty}^\infty \int_{t-1}^t  \int_{-\infty}^\infty   \ell_t(\theta; s+v, c+v,x, v )\, f(s|x,0) \, ds \, dv \,  g(c |x,0) dc\,  \right] \,. 
\end{align*}

We do not know $g(c|x,v)$ but we know that $g(c|x,v) = g(c-\epsilon|x, v-\epsilon)$ for all $\epsilon \in \mathbb{R}$. For instance, $g(c|x,v) = g(c-v|x,0)$ and, the same is satisfied for $f$. Then, we change the variable $v\in [0,t]$ in $w=v-(t-1)$:
\begin{align*}
\int_{t-1}^t     \ell_t(\theta; s+v,c+v,x,v )   \, dv  =  \int_{0}^1     \ell_1(\theta; s+w,c+w,x, w )   \, dw\,.
\end{align*}
We obtain 
\begin{align*}
&\mathbb{E} \left[ \sum_{i=N_{t-1}}^{N_t} \ell_t(\theta;t_i,c_i,x_i,\tau_i) \right]\\ & \hspace{1cm} = \lambda \mathbb{E} \left[ \int_{-\infty}^\infty \int_{0}^1  \int_{-\infty}^\infty   \ell_1(\theta; s+w,c+w,x, w) f(s|x,0) ds  \, dw  \,  g(c |x,0) dc\, \right].
\end{align*}
Considering the integral on the time $s$
\begin{align*}
&\int_{-\infty}^\infty    \ell_1(\theta;s+w,c+w,x, w) f(s|x,0) ds\\
& \hspace{3cm} =\int_{0}^\infty\ell_1(\theta; s+w,c+w,x, w)e^{{\theta^*}^Tx} \exp(-s e^{{\theta^*}^Tx})ds\,,
\end{align*}
we obtain
\begin{align*}
&-\int_{(-w)_+}^{(1-w)\wedge c}\theta^Txe^{{\theta^*}^Tx} \exp(-s e^{{\theta^*}^Tx})ds\\
& \hspace{5cm}=-\theta^Tx\mathbb P((-w)_+<T\le (1-w)\wedge c \vert \tau = 0)\,,
\end{align*}
and
\begin{align*}
&\int_{(1-w)\wedge c}^{\infty}((c+w)\wedge(s+w)\wedge 1-w_+)_+ e^{{\theta}^Tx}e^{{\theta^*}^Tx} \exp(-s e^{{\theta^*}^Tx})ds\\& \hspace{1.7cm}=e^{{\theta}^Tx}((c+w)\wedge 1-w_+)_+\mathbb P(T\ge (1-w)\wedge c\vert \tau = 0)\,,
\end{align*}
and
\begin{align*}
&\int_{(-w)_+}^{(1-w)\wedge c}((s+w)-w_+)e^{{\theta}^Tx}e^{{\theta^*}^Tx} \exp(-s e^{{\theta^*}^Tx})ds\\
&=e^{{\theta}^Tx}\Big(-((s+w)-w_+)\exp(-s e^{{\theta^*}^Tx})\mid_{(-w)_+}^{(1-w)\wedge c}+\int_{(-w)_+}^{(1-w)\wedge c}\exp(-s e^{{\theta^*}^Tx})ds\Big)\\
& =-e^{{\theta}^Tx}((c+w)\wedge 1-w_+)_+\mathbb P(T\ge (1-w)\wedge c \vert \tau = 0)\\
& \hspace{2.5cm}+\exp((\theta-\theta^*)^T x)\mathbb P((-w)_+<T\le (1-w)\wedge c \vert \tau = 0)\,.
\end{align*}
All in all we obtain\begin{align*}
&\int_{0}^\infty\ell_1(\theta; s+w,c+w,x, w)e^{{\theta^*}^Tx} \exp(-s e^{{\theta^*}^Tx})ds\\ & \hspace{1.7cm}=(e^{(\theta-\theta^*)^T x}-\theta^Tx)\mathbb P((-w)_+<T\le (1-w)\wedge c \vert \tau = 0)\, ,
\end{align*}
thus we have
\begin{align*}
\mathbb{E}& \left[ \displaystyle \sum_{i=N_{t-1}}^{N_t} \ell_t(\theta;t_i,c_i,x_i,\tau_i) \right]\\
& =\lambda \mathbb{E} \left[ (e^{(\theta-\theta^*)^T x}-\theta^Tx)\int_{-\infty}^\infty \int_{0}^1  \mathbb P((-w)_+<T\le (1-w)\wedge c\vert \tau = 0)  dw\,  g(c |x,0) dc\, \right] \\
& =\lambda \mathbb{E} \Big[(e^{(\theta-\theta^*)^T x}-\theta^Tx)\mathbb E \Big[\int_{0}^1  \mathds 1\{(-w)_+<T\le (1-w)\wedge C\}  dw \Big\vert x, \tau = 0\Big]\Big]\\
& =  \lambda \mathbb{E} \Big[(e^{(\theta-\theta^*)^T x}-\theta^Tx)\mathbb E \Big[\int_{0\vee -T}^{1-T} \mathds{1}\{T \leq C\}   dw\Big \vert x, \tau = 0\Big]\Big]\\
& =  \lambda \mathbb{E} \Big[(e^{(\theta-\theta^*)^T x}-\theta^Tx)\mathbb E \Big[ \mathds{1}\{T\leq C\} ( (1-T) - 0 \vee -T)_+  \Big \vert x, \tau = 0\Big]\Big]\\
& =  \lambda \mathbb{E} \Big[(e^{(\theta-\theta^*)^T x}-\theta^Tx)\mathbb E \Big[ \mathds{1}\{T\leq C\} (1-T)_+ \Big \vert x, \tau = 0\Big]\Big].
\end{align*}
Replacing in Equation \eqref{eq:sep} we have  \begin{align*}
\mathbb{E}_{t-1}[\ell_t(\theta)] &= \lambda \mathbb{E} \Big[(e^{(\theta-\theta^*)^T x}-\theta^Tx)\mathbb E \Big[ \mathds{1}\{T\leq C\} (1-T)_+ \Big \vert x,\tau = 0\Big]\Big]\\ & \hspace{2cm}  + \sum_{i=1}^{N_{t-1}} \mathbb{E}_{t-1} \left[  \ell_t(\theta;t_i,c_i,x_i,\tau_i) \vert x_i,\tau_i \right]\, . 
\end{align*}
To calculate the second term, we note that we know $u_i$ if $u_i \leq t-1$ and in this case $\ell_t(\theta;t_i,c_i,x_i,\tau_i) = 0$, therefore, we consider only the individuals $i$ such that $u_i > t-1$. The sum becomes  \begin{align*} & \sum_{i=1}^{N_{t-1}} \mathbb{E}_{t-1} \left[  \ell_t(\theta;t_i,c_i,x_i,\tau_i) \vert x_i,\tau_i \right]\\& \hspace{2cm} = \sum_{ \substack{i: \{ u_i >t-1\}\\ i: \{ \tau_i \leq t-1\}  } } \int_{-\infty}^{\infty} \int_{-\infty}^{\infty}  \ell_t(\theta;s,c,x_i,\tau_i) g(c|x_i,\tau_i)f(s|x_i,\tau_i)dsdc\, , \end{align*} which, following the calculations of the first term we obtain
\begin{align*}  
& \sum_{i=1}^{N_{t-1}} \mathbb{E}_{t-1} \left[  \ell_t(\theta;t_i,c_i,x_i,\tau_i) \vert x_i,\tau_i \right]\\ 
& = \sum_{ \substack{i: \{ u_i >t-1\}\\ i: \{ \tau_i \leq t-1\} } }   \left(e^{(\theta-\theta^*)^{\top}x_i} - \theta^{\top}x_i \right) \mathbb{P}(t-1+\tau_i < T \leq \tau_i + t \wedge C \vert x=x_i,\tau_i, \tau=0)\,,
\end{align*} 
where $(T,C,x,\tau)$ is independent of $(t_i,c_i,x_i,\tau_i)$ for every $i\ge 1$.
Let us notice that $\tau_i$ and $x_i$, which we suppose are observed at the same time as $\tau_i$, are known at time $t-1$. Replacing this term in Equation \eqref{eq:sep} leads to
\begin{align*}
&\mathbb{E}_{t-1}[\ell_t(\theta)]\\
&= \lambda \mathbb{E} \Big[(e^{(\theta-\theta^*)^T x}-\theta^Tx) \mathds{1}\{T\leq C\} (1-T)_+ \mid \tau =0\Big]\\ 
& \hspace{.5cm} + \sum_{ \substack{i: \{ u_i >t-1\}\\ i: \{ \tau_i \leq t-1\}  } }   \left(e^{(\theta-\theta^*)^{\top}x_i} - \theta^{\top}x_i \right) \mathbb{P}(t-1+\tau_i < T \leq \tau_i + t \wedge C \vert x_i,\tau_i,\tau = 0), \end{align*}
that finalizes the proof.
\end{proof}

We define  \begin{equation}\label{J} J(\theta) := \lambda \mathbb{E} \Big[(e^{(\theta-\theta^*)^{\top} x}-\theta^Tx) \mathds{1}\{T\leq C\} (1-T)_+ \Big \vert x, \tau = 0\Big]\, ,\end{equation}
and we are ready to show the strong convexity of the risk function that we give in the following proposition.
\begin{proposition}\label{lowerbnd} The risk function satisfies \[ \nabla^2 L_t(\theta) \succcurlyeq \lambda e^{-Dx_{\infty}}\mathbb E [xx^{\top} \mathds 1\{ T \leq C \} (1-T)_+ \vert \tau = 0 ], \qquad \forall \theta \in \Theta, t= 1,\ldots,n\, . \] Therefore, under Assumption~\ref{ass:strongconv} $L_t$ is $\mu$-strongly convex for $\mu = \lambda e^{-Dx_{\infty}}A$. 
\end{proposition}

\begin{proof}
Lemma~\ref{risk} gives us an expression of the risk $L_t=:J+R_t$ with $R_t$ some random convex function. By convexity $\nabla^2 L_t(\theta) \succcurlyeq \nabla^2 J(\theta)$, $\theta\in \Theta$, and therefore, it is enough to bound the hessian of the first term $J$. We calculate \begin{equation*}
    \nabla J(\theta) = \lambda \mathbb{E}\Big[ (e^{(\theta - \theta^*)^{\top}x}-1)x \mathds{1}\{T\leq C\}(1-T)_+ \big \vert x,\tau = 0 \Big]
\end{equation*}
and \begin{equation*} \nabla^2 J(\theta) = \lambda \mathbb{E}\Big[ e^{(\theta - \theta^*)^{\top}x}xx^{\top} \mathds{1}\{T\leq C\}(1-T)_+ \big \vert x,\tau = 0 \Big]\, .
\end{equation*}

Let us notice that $e^{(\theta - \theta^*)^{\top}x} \geq e^{-Dx_{\infty}}$ and then \[\nabla^2J(\theta) \succcurlyeq \lambda e^{-Dx_{\infty}}\mathbb E [xx^{\top} \mathds 1\{ T \leq C \} (1-T)_+ \vert x, \tau = 0 ]\, ,\]
which due to Assumption~\ref{ass:strongconv} concludes the proof.
\end{proof}

\subsection{Proof of Theorem~\ref{thr:stcregret} }
\begin{proof}
First of all, we remind that Proposition~\ref{upperbnd} implies \textbf{(H3)} with $G = G_1 =G_2$, but this bound for the gradients is satisfied with probability $1-\varrho$ instead of almost surely and therefore we cannot claim that \textbf{(H3)} is always fulfilled. 
But there is a problem with this definition because \textbf{(H3)} considers all $t=1,2,\ldots$ and, in order to have a $\mathcal F_t$-measurable function we need to define a time dependent \textbf{(H3)}$_t$: 

\textbf{(H3)}$_t$ For $t+1 \geq s\geq 1$ the gradients $\nabla \ell_s(\theta_s)$, satisfy for $G>0$ and all $k \geq 1,$ and $ \theta \in \Theta$: 
    \begin{align*}
    \mathbb E_{s-1}[(\nabla \ell_s(\theta_s)^{\top} (\theta_s - \theta))^{2k}]& \le k! (G D)^{2(k-1)} \mathbb E_{s-1}[ (\nabla \ell_s(\theta_s)^{\top}(\theta_s - \theta))^2] \,, \\
     \mathbb E_{s-1} [ || \nabla \ell_s (\theta) ||^{2k} ]& \le k! G^{2(k-1)}\mathbb E_{s-1} [ || \ell_s (\theta_s) ||^2]  \,,\\
     \mathbb E_{s-1} [ || \nabla \ell_t (\theta) ||^2 ]&  \le G^2 \,.  
\end{align*}

We define $\Omega_t = \{ (y_{is},x_i,\tau_i,u_i)_{s \leq t} \text{ such that \textbf{(H3)}$_t$ is satisfied} \} $ for all $t=1,2,\ldots $ sand we check that $\Omega_t$ is $\mathcal F_t$-measurable. Next, for all $t=1,2,\ldots$ we define the auxiliary loss function 
\[ \hat{\ell_t}(\theta) = \ell_t (\theta) \mathds{1}\{ \Omega_{t-1} \}\, , \]
which is $\mathcal F_{t}$-measurable. Let us notice that we need to define $\Omega_t$ using \textbf{(H3)}$_t$ instead of the inequality of Proposition~\ref{upperbnd} to preserve the past dependency and the measurability. We prove that the function $\hat{\ell_t}$ satisfies the conditions \textbf{(H1)}, \textbf{(H2)} and \textbf{(H3)}.

First of all, \textbf{(H1)} is still verified because the indicator function does not depend on $\theta$. Secondly, if \textbf{(H3)}$_t$ is not realized then the function $\hat{\ell_t}$ is zero and all the bounds hold. Thirdly, if \textbf{(H3)}$_t$ is realized, $\ell_t$ satisfies the inequalities of \textbf{(H3)} and $\hat{\ell_t} = \ell_t$ by definition. Then the bounds in \textbf{(H3)} are also true for $\hat{\ell_t}$, concluding that $\hat{\ell_t}$ satisfies the inequalities of \textbf{(H3)} for all $t=1,2,\ldots$. Finally, it remains to prove \textbf{(H2)}.

By $\mathcal F_{t-1}$-measurability of $\Omega_{t-1}$ we calculate for $\theta \in \Theta$:  
\[ \mathbb E_{t-1} [\nabla \hat{\ell _t}(\theta) \nabla \hat{\ell_t} (\theta)^{\top}] = \mathds{1}\{\Omega_{t-1}\} \mathbb E_{t-1} [ \nabla \ell _t(\theta) \nabla\ell_t (\theta)^{\top}]\, .\]
If \textbf{(H3)}$_t$ is not realized, $\mathds{1}\{\Omega_{t-1}\} = 0$ and so \eqref{con:1} is true for any constant $\gamma \geq 0$. If \textbf{(H3)}$_t$ is realized, $\mathds{1}\{\Omega_{t-1}\} = 1$ and there exist $G >0$ such that:
\begin{align*}
\mathbb E_{t-1} [\nabla \hat{\ell _t}(\theta) \nabla \hat{\ell_t} (\theta)^{\top}] = \mathbb E_{t-1} [ \nabla \ell _t(\theta) \nabla\ell_t (\theta)^{\top}]  \preccurlyeq G^2 \mathcal I_d\, .
\end{align*}
This, together with the strong convexity of Proposition~\ref{lowerbnd} give us the hypothesis of Proposition~\ref{prop:olivier} assuring the stochastic exp-concavity for $\gamma = \lambda e^{-Dx_{\infty}}A/G^2$ and concluding~\textbf{(H2)}. Now, we have that $\hat{\ell_t}$ satisfies all the conditions of Theorem~\ref{thr:olivier} assuring the logarithmic stochastic regret bound of ONS. 

To study the stochastic regret bound we need also to define for all $t=1,2,\ldots$ the risk function $\hat{L_t}(\theta) = \mathbb E_{t-1}[ \hat{\ell_t}(\theta)]$ and we notice that as $\mathds{1}\{\Omega_{t-1}\}$ is $\mathcal F_{t-1}$-measurable:
\[
\hat{L_t}(\theta) = \mathds{1}\{\Omega_{t-1}\} \mathbb E_{t-1}[\ell_t(\theta)] = \mathds{1}\{\Omega_{t-1}\} L_t(\theta) \,, \qquad \theta\in\Theta .
\]

Now, it remains to prove that ONS has logarithmic stochastic regret also for $L_t$ and therefore, we calculate for every $n\ge1, t=1,\ldots,n$,\\ $\theta^* \in \arg\min_{\theta \in \Theta}  \sum_{t=1}^n L_{t}(\theta) $ and $\theta_t$ the prediction of ONS at time $t$:
\begin{align*}
    \mathbb P \left[\sum_{t=1}^n  L_t(\theta_t) - L_t(\theta) >   \mathcal B (n) \right] & = \mathbb P \left[\sum_{t=1}^n  L_t(\theta_t) - L_t(\theta^*) >  \mathcal B (n), \bigcap_{t \geq 2} \Omega_{t-1} \right]\\
    & + \mathbb P \left[\sum_{t=1}^n  L_t(\theta_t) -  L_t(\theta^*) >  \mathcal B (n), \left(\bigcap_{t \geq 2} \Omega_{t-1}\right)^c \right]  \\
    & \leq \mathbb P \left[\sum_{t=1}^n \left( L_t(\theta_t) -  L_t(\theta^*)\right) \mathds{1}\{\Omega_{t-1}\} >  \mathcal B (n) \right]\\
    & \hspace{2cm} + \mathbb P \left[ \left(\bigcap_{t \geq 2} \Omega_{t-1}\right)^c \right]\, , 
\end{align*}

where $\mathcal B(n)$ is the stochastic regret bound for $\hat L(\theta)$ of Theorem~\ref{thr:olivier} which we remind:
\begin{align*}
     \mathcal B (n)&= \frac{3}{2\gamma} \left(1+d\log\left(1+\frac{2(\gamma D)^2 G^2\left(n + \log(\varrho^{-1}) \right) }{9} \right) \right) \\
    & \hspace{6cm} + \left( \frac{4 \gamma (G D)^2}{9} + \frac{18}{\gamma} \right) \log(\varrho^{-1})\,.
    \end{align*}

Plugging in $ \mathcal B (n)$ the specific values of $\gamma$, $G$ and $\mu$ found in Propositions~\ref{prop:olivier},~\ref{upperbnd}, and~\ref{lowerbnd}, respectively, \[\gamma = \frac{\mu}{G^2} = \frac{\lambda e^{-Dx_{\infty}}A}{\left(32 e^{Dx_{\infty}}(4\lambda +1 +\log(2/\varrho))(1+e^{Dx_{\infty}})x_{\infty}\right)^2 }  \, , \]
we obtain the regret bound
\begin{align}
     Risk_n&\leq \frac{3G^2e^{Dx_{\infty}}}{2\lambda A} \left(1+d\log\left(1+\frac{2(\lambda A D)^2\left(n + \log(\varrho^{-1}) \right) }{9G^2e^{2Dx_{\infty}}} \right) \right) \nonumber \\
    & \hspace{4cm} + \left( \frac{4 \lambda A D^2}{9e^{Dx_{\infty}}} + \frac{18Ge^{Dx_{\infty}}}{\lambda A} \right) \log(\varrho^{-1})\,. \label{stc:explicitbound}
\end{align}

Then, because of Theorem~\ref{thr:olivier}, this bound holds with probability $3\varrho$ and as \[\mathbb P \left[ \left(\bigcap_{t \geq 2} \Omega_{t-1}\right)^c \right] \leq \varrho\,,\] 
we have:   
\[
\mathbb P \left[\sum_{t=1}^n L_t(\theta_t) -  L_t(\theta) > \mathcal O(\log(n/\varrho)) \right] \leq 4\varrho\, ,
\]
and thus, with probability $1-4\varrho$, ONS algorithm has logarithmic stochastic regret.
\end{proof}

\subsection{Proof of Corollary~\ref{cor:stc}}
\begin{proof}
Due to the $\mu$-strong convexity of $L_t(\theta)$ proved in Proposition~\ref{lowerbnd} we have for all $t=1,2,\ldots$: 
\[
\mu ||\theta_t - \theta^*||^2 \leq \nabla L_t(\theta^*)^{\top}(\theta_t - \theta^*) + \mu ||\theta_t - \theta^*||^2 \leq
L_t(\theta_t) - L_t(\theta^*) \, , 
\]
where the first inequality is true because $\nabla L_t(\theta^*)^{\top}(\theta_t - \theta^*) \geq 0$. Then, because of Theorem~\ref{thr:stcregret}:
\[
\sum_{t=1}^n ||\theta_t - \theta^*||^2  \leq \frac{1}{\mu} \sum_{t=1}^n
L_t(\theta_t) - L_t(\theta^*) \leq \frac{1}{\mu} \mathcal B (n) \, , 
\]
and remembering that $\mu = \lambda e^{-Dx_{\infty}}A$, the  bound is:
\begin{align}
    \frac{1}{\mu} \mathcal B(n)  & =  \frac{3G^2e^{2Dx_{\infty}}}{2\lambda^2 A^2} \left(1+d\log\left(1+\frac{2(\lambda A D)^2\left(n + \log(\varrho^{-1}) \right) }{9G^2e^{2Dx_{\infty}}} \right) \right) \nonumber \\
    & \hspace{4cm} + \left( \frac{4 \lambda A D^2}{9e^{Dx_{\infty}}} + \frac{18Ge^{Dx_{\infty}}}{\lambda A} \right) \log(\varrho^{-1})\,, \label{cor:explicitbound}
\end{align}
 which is $\mathcal O (\log(n/\varrho))$. We conclude the convergency of $\theta_t$ to $\theta^*$ and then: 
\[
||\bar \theta_n - \theta^*||^2 \leq \frac{1}{n} \sum_{t=1}^n || \theta_t - \theta^*||^2 - \frac{1}{n} \sum_{t=1}^n ||\theta_t - \bar \theta_n||^2 \leq \frac{1}{\mu} \mathcal O (\log(n/\varrho)/n)\, ,
\]
concluding the convergency of $\bar \theta_n$ to $\theta^*$. 
\end{proof}

\section{SurvONS} \label{AppD}
\subsection{Proof of Lemma~\ref{lem:aux}}
\begin{proof}
    We first compute 
    \begin{align*} \nabla \hat{\ell}_{t,\gamma}(\theta_1)  &= \nabla \ell_t(\hat\theta) + \gamma \left( \nabla \ell_t(\hat\theta)(\theta_1-\hat\theta) \right) \nabla \ell_t(\hat\theta)\\& = \left( 1 + \gamma \nabla \ell_t(\hat\theta)(\theta_1-\hat\theta) \right) \nabla \ell_t(\hat\theta)\,.
    \end{align*}
We need to show that there exists $\hat{\gamma}>0$ such that
\begin{align*}
    \hat{\ell}_{t,\gamma}(\theta_2) \geq \hat{\ell}_{t,\gamma}(\theta_1) + \nabla \hat{\ell}_{t,\gamma}(\theta_1)(\theta_2- \theta_1) +  \frac{\hat \gamma}{2} \left(\nabla \hat{\ell}_{t,\gamma} (\theta_1)(\theta_2-\theta_1) \right)^2  
\end{align*}
and if we replace $\hat{\ell}_{t,\gamma}$ this inequality is equivalent to 
  \begin{align*} &\ell_t(\hat\theta)+\nabla \ell_t(\hat\theta)(\theta_2-\hat\theta)+\frac{\gamma}{2} \left( \nabla \ell_t(\hat\theta) (\theta_2 - \hat\theta) \right)^2\\ & \geq \ell_t(\hat\theta) + \nabla \ell_t(\hat\theta) (\theta_1 - \hat\theta) + \frac{\gamma}{2} \left( \nabla \ell_t(\hat\theta) (\theta_1 - \hat\theta)\right)^2 \\ & \hspace{0.5cm}+\left(1 + \gamma \nabla \ell_t(\hat\theta)(\theta_1- \hat\theta)\right)\nabla \ell_t(\hat\theta)(\theta_2-\theta_1)\\ &\hspace{1cm}  + \frac{\hat{\gamma}}{2}\left( (1+ \gamma \nabla \ell_t(\hat\theta) (\theta_1-\hat\theta) ) \nabla \ell_t(\hat\theta) (\theta_2-\theta_1)\right)^2. \end{align*} Grouping, this requirement becomes \begin{align*}
        \frac{\gamma}{2} \left( \nabla \ell_t(\hat\theta)(\theta_2-\hat\theta)\right)^2 & \geq \frac{\gamma}{2} \left( \nabla \ell_t(\hat\theta)(\theta_1-\hat\theta) \right)^2\\&\hspace{0.5cm} + \gamma \nabla \ell_t(\hat\theta)(\theta_1-\hat\theta)\nabla \ell_t(\hat\theta)(\theta_2-\theta_1)\\& \hspace{1cm}  + \frac{\hat{\gamma}}{2}\left( (1+ \gamma \nabla \ell_t(\hat\theta) (\theta_1-\hat\theta) ) \nabla \ell_t(\hat\theta) (\theta_2-\theta_1)\right)^2,
    \end{align*} which is \begin{equation*}
        0 \geq \left( \frac{\hat{\gamma}}{2}(1 +\gamma \nabla \ell_t(\hat\theta)(\theta_1-\hat\theta))^2 + \frac{\gamma}{2} \right) \left(\nabla \ell_t(\hat\theta)(\theta_2-\theta_1)\right)^2.
    \end{equation*} To satisfy this inequality we need \begin{equation*}
        \hat{\gamma} \leq \frac{\gamma}{(1+ \gamma \nabla \ell_t(\hat\theta)(\theta_1-\hat\theta))^2},
    \end{equation*} which is true for the choice $\hat{\gamma} = \frac{\gamma}{(1+\gamma D||\nabla \ell_t(\hat\theta)||)^2}$ and this concludes the proof.
\end{proof}

\subsection{Proof of Theorem~\ref{regbnd}}
\begin{proof} 
At each iteration $t$ we consider $\hat{\theta}_t$ the prediction of SurvONS and $\theta_t(\gamma)$ the prediction of ONS with $\gamma \in \Gamma$. We define the directional derivative lower bound function as in Equation \eqref{eq:auxfunc} \[ \hat{\ell}_{t,\gamma_t} (\theta) = \ell_t(\hat{\theta}_t) + \nabla \ell_t(\hat{\theta}_t)(\theta - \hat{\theta}_t) + \frac{\gamma_t}{2} \left( \nabla \ell_t(\hat{\theta}_t)(\theta - \hat{\theta}_t)\right)^2.\]
Let us notice that $\gamma \leq \frac{1}{4GD}$ and $\hat{\ell}_{t,\gamma_t}(\theta) \leq \ell_t(\theta)$ for all $\theta$.

We take $ \theta^* \in \arg\min_{\theta \in \Theta}  \sum_{t=1}^n \ell_t(\theta)$ and we can upper-bound the regret for any $\gamma \in \Gamma$ \begin{align*}
        Regret_n & = \sum_{t=1}^n \ell_t(\hat{\theta}_t) - \ell_t(\theta^*) \\ 
        & \leq \sum_{t=1}^n \hat{\ell}_{t,\gamma_t}(\hat{\theta}_t) - \hat{\ell}_{t,\gamma_t}(\theta^*)\\
        & = \sum_{t=1}^n \hat{\ell}_{t,\gamma_t}(\hat{\theta}_t) - \hat{\ell}_{t,\gamma_t}(\theta_t(\gamma))  + \hat{\ell}_{t,\gamma_t}(\theta_t(\gamma)) - \hat{\ell}_{t,\gamma_t}(\theta^*)\\
        & = \sum_{t=1}^n \nabla \ell_t(\hat{\theta}_t)(\hat{\theta}_t - \theta_t(\gamma)) - \sum_{t=1}^n \frac{\gamma_t}{2} \left( \nabla \ell_t(\hat{\theta}_t)(\hat{\theta}_t - \theta_t(\gamma))\right)^2 \\
        & \hspace{2cm} +\sum_{t=1}^n \hat{\ell}_{t,\gamma_t}(\theta_t(\gamma)) - \hat{\ell}_{t,\gamma_t}(\theta^*)\,.
    \end{align*}
We upper-bound the first term using the regret-bound of BOA~\cite{wintenberger2017optimal} which works for $\gamma \leq \frac{1}{4GD}$ \[ \sum_{t=1}^n \nabla \ell_t(\hat{\theta}_t)(\hat{\theta}_t - \theta_t(\gamma)) \leq \frac{\log(K)}{ \gamma} + 2 \gamma \sum_{t=1}^n \left(\nabla \ell(\hat{\theta}_t)(\hat{\theta}_t - \theta_t(\gamma))  \right)^2\, .  \]
Therefore, the regret is bounded by
\begin{align*}  Regret_n &\leq\frac{\log(K)}{ \gamma} + \sum_{t=1}^n \left(\frac{4\gamma - \gamma_t}{2}\right) \left(\nabla \ell(\hat{\theta}_t)(\hat{\theta}_t - \theta_t(\gamma))  \right)^2 \\
&\hspace{6cm}+ \sum_{t=1}^n \hat{\ell}_{t,\gamma_t}(\theta_t(\gamma)) - \hat{\ell}_{t,\gamma_t}(\theta^*). \end{align*}
We consider the surrogate losses $\hat \ell_{t,\hat\gamma_t}$ for $t=1,2,\ldots$ and $\hat \gamma_t = 4 \max\{ \gamma, \gamma_t/4\}$
\[ \hat{\ell}_{t,\hat \gamma_t}(\theta) = \ell_t(\hat{\theta}_t) + \nabla \ell_t(\hat{\theta}_t)(\theta - \hat{\theta}_t) + 2 \max\{\gamma,\frac{\gamma_t}{4}\} \left( \nabla \ell_t(\hat{\theta}_t)(\theta - \hat{\theta}_t)\right)^2\, ,  \] 
and we write the last term of the regret bound \begin{align*} &\sum_{t=1}^n \hat{\ell}_{t,\gamma_t}(\theta_t(\gamma)) - \hat{\ell}_{t,\gamma_t}(\theta^*)\\ &\hspace{1cm} = \sum_{t=1}^n  \left(\hat{\ell}_t(\theta_t(\gamma);\gamma) - \hat{\ell}_t(\theta^*;\gamma)\right) + \sum_{t=1}^n \left( \hat{\ell}(\theta_t(\gamma)) - \hat{\ell}_t(\theta_t(\gamma);\gamma) \right)\\
& \hspace{2cm} - \sum_{t=1}^n \left(\hat{\ell}_t(\theta^*) - \hat{\ell}_t(\theta^*;\gamma) \right)\\
&\hspace{1cm} =  \sum_{t=1}^n  \left(\hat{\ell}_t(\theta_t(\gamma);\gamma) - \hat{\ell}_t(\theta^*;\gamma)\right)\\ 
& \hspace{2cm} - \sum_{t=1}^n \frac{(4\gamma - \gamma_t)_+}{2} \left( \nabla \ell_t(\hat{\theta}_t)(\hat{\theta}_t - \theta_t(\gamma)) \right)^2\\
& \hspace{4cm} + \sum_{t=1}^n  \frac{(4\gamma - \gamma_t)_+}{2} \left( \nabla \ell_t(\hat{\theta}_t)(\hat{\theta}_t - \theta^*) \right)^2\, .  \end{align*}
We substitute this expression in the regret bound \begin{align*} Regret_n &\leq \frac{ \log(K)}{ \gamma} + \sum_{t=1}^n  \left(\hat{\ell}_t(\theta_t(\gamma);\gamma) - \hat{\ell}_t(\theta^*;\gamma)\right) \\ & \hspace{1cm}- \sum_{t=1}^n \frac{( \gamma_t - 4 \gamma)_+}{2} \left( \nabla \ell_t(\hat{\theta}_t)(\hat{\theta}_t - \theta_t(\gamma)) \right)^2 \\&\hspace{2cm} + \sum_{t=1}^n  \frac{(4\gamma - \gamma_t)_+}{2} \left( \nabla \ell_t(\hat{\theta}_t)(\hat{\theta}_t - \theta^*) \right)^2\, . \end{align*}
Now, we note that by Lemma~\ref{lem:aux} we have \begin{align*} \hat{\ell}_t(\theta_t(\gamma);\gamma) - \hat{\ell}_t(\theta^*;\gamma) &\leq \nabla \hat{\ell}_t(\theta_t(\gamma);\gamma)(\theta_t(\gamma) - \theta^*)\\ &\hspace{2cm} - \frac{\hat{\gamma}_t}{2}\left( \nabla \hat{\ell}_t(\theta_t(\gamma);\gamma)(\theta_t(\gamma) - \theta^*) \right)^2\, ,  \end{align*} where we can write $\max \{\gamma,\gamma_t/4\} = \gamma + 
\left(\gamma_t/4 - \gamma \right)_+$ and get \[\hat{\gamma}_t = \frac{4(\gamma+(\gamma_t/4 - \gamma)_+ )}{(1+4(\gamma + (\gamma_t/4 - \gamma)_+)(\nabla \ell_t(\hat{\theta}_t) (\theta_t(\gamma) - \hat{\theta}_t)))^2} \geq \gamma \, . \]  
Therefore, we can apply the regret bound of ONS which yields to \begin{align}\label{ons:ineq} \sum_{t=1}^n \hat{\ell}_t(\theta_t(\gamma);\gamma) - \hat{\ell}_t(\theta^*;\gamma) & \leq \frac{5d\log(n)}{\gamma} + \frac{\gamma}{2} \sum_{t=1}^n \left(\nabla \hat{\ell}_t(\theta_t(\gamma);\gamma)(\theta_t(\gamma) - \theta^*)\right)^2\nonumber\\& \hspace{2cm} -  \sum_{t=1}^n \frac{\hat{\gamma}_t}{2}\left(\nabla \hat{\ell}_t(\theta_t(\gamma);\gamma)(\theta_t(\gamma) - \theta^*)\right)^2 \, . \end{align} But, since \[ \nabla \hat{\ell}_t(\theta_t(\gamma);\gamma) = \left(1+4( \gamma + 
\left(\gamma_t/4 - \gamma \right)_+) \nabla\ell_t(\hat{\theta}_t)(\theta_t(\gamma) - \hat{\theta}_t)  \right)\nabla\ell_t(\hat{\theta}_t)\, ,\] we can write \begin{align*}& \left(\nabla \hat{\ell}_t(\theta_t(\gamma);\gamma)(\theta_t(\gamma) - \theta^*) \right)^2\\& \hspace{1cm} = \left(1+4( \gamma + 
\left(\gamma_t/4 - \gamma \right)_+) \nabla\ell_t(\hat{\theta}_t)(\theta_t(\gamma) - \hat{\theta}_t)  \right)^2 \left(\nabla \ell_t(\hat{\theta}_t)(\theta_t(\gamma) - \theta^*)\right)^2,\end{align*} which yields to \begin{equation}\label{eqaux1}\hat{\gamma}_t \left(\nabla \hat{\ell}_t(\theta_t(\gamma);\gamma)(\theta_t(\gamma) - \theta^*) \right)^2 = 4( \gamma + 
\left(\gamma_t/4 - \gamma \right)_+)\left(\nabla \ell_t(\hat{\theta}_t)(\theta_t(\gamma) - \theta^*)\right)^2. \end{equation}
Using the assumption $4(\gamma + (\gamma_t/4 - \gamma)_+ )\leq 1/GD$ we can also get
\begin{equation}\label{eqaux2}\left(\nabla \hat{\ell}_t(\theta_t(\gamma);\gamma)(\theta_t(\gamma) - \theta^*) \right)^2 \leq 4\left(\nabla \ell_t(\hat{\theta}_t)(\theta_t(\gamma) - \theta^*)\right)^2. \end{equation}

Therefore, plugging  \eqref{eqaux1} and \eqref{eqaux2} in \eqref{ons:ineq} we get \begin{align*}\sum_{t=1}^n \hat{\ell}_t(\theta_t(\gamma);\gamma) - \hat{\ell}_t(\theta^*;\gamma) & \leq \frac{5d\log(n)}{\gamma} + \frac{4\gamma}{2} \sum_{t=1}^n \left(\nabla \ell_t(\hat{\theta}_t)(\theta_t(\gamma) - \theta^*)\right)^2\\&  -  2\sum_{t=1}^n (\gamma +(\gamma_t/4 -\gamma)_+)\left(\nabla \ell_t(\hat{\theta}_t)(\theta_t(\gamma) - \theta^*)\right)^2\\ & = \frac{5d\log(n)}{\gamma}\\&\hspace{0.8cm} - 2 \sum_{t=1}^n (\gamma_t/4 -\gamma)_+\left(\nabla \ell_t(\hat{\theta}_t)(\theta_t(\gamma) - \theta^*)\right)^2.\end{align*} 
Thus, the regret bound becomes 
\begin{align*} &Regret_n  \leq \frac{ 2 \log(K) + 5d\log(n)}{\gamma}\\
& -2\sum_{t=1}^n  (\gamma_t/4 - \gamma)_+ \left( (\nabla \ell_t (\hat{\theta}_t)(\theta_t(\gamma) - \theta^*))^2 + (\nabla \ell_t(\hat{\theta}_t)(\hat{\theta}_t - \theta_t(\gamma)))^2  \right)\\
&  + \sum_{t=1}^n \frac{ (4\gamma - \gamma_t)_+ }{2}(\nabla \ell_t(\hat{\theta}_t)(\hat{\theta}_t - \theta^*))^2, \end{align*}
and as $(4\gamma - \gamma_t)_+ = 4\gamma + 4(\gamma_t/4 - \gamma)_+ -\gamma_t$, we can regroup and get 
\begin{align*}Regret_n \leq & \frac{ 2 \log(K) + 5d\log(n)}{\gamma}\\
& +2\sum_{t=1}^n  (\gamma_t/4 - \gamma)_+ \Big((\nabla \ell_t(\hat{\theta}_t)(\hat{\theta}_t - \theta^*))^2- (\nabla \ell_t (\hat{\theta}_t)(\theta_t(\gamma) - \theta^*))^2\\
& \hspace{6cm} - (\nabla \ell_t(\hat{\theta}_t)(\hat{\theta}_t - \theta_t(\gamma)))^2  \Big)\\
& + \sum_{t=1}^n \frac{ 4\gamma - \gamma_t }{2}(\nabla \ell_t(\hat{\theta}_t)(\hat{\theta}_t - \theta^*))^2\\  =& \frac{ 2 \log(K) + 5d\log(n)}{\gamma}\\
& -4\sum_{t=1}^n  (\gamma_t/4 - \gamma)_+ \Big((\nabla \ell_t (\hat{\theta}_t)(\theta_t(\gamma) - \theta^*))(\nabla \ell_t(\hat{\theta}_t)(\theta_t(\gamma)-\hat{\theta}_t)) \Big)\\
& + \sum_{t=1}^n \frac{ 4\gamma - \gamma_t }{2}(\nabla \ell_t(\hat{\theta}_t)(\hat{\theta}_t - \theta^*))^2 \end{align*} 

\end{proof}

\end{document}